\UseRawInputEncoding
\pdfoutput=1
\documentclass[conference]{IEEEtran}
\IEEEoverridecommandlockouts
\usepackage{cite}
\usepackage{amsmath,amssymb,amsfonts}
\usepackage{algorithmic}
\usepackage{graphicx}
\usepackage{tabularx}
\usepackage{booktabs}
\usepackage{textcomp}
\usepackage{xcolor}
\usepackage{tabularx}
\usepackage{subfigure}
\usepackage{amsthm}

\newtheorem{theorem}{Theorem}
\newtheorem{property}{Property}
\newtheorem{proposition}{Proposition}
\newtheorem{lemma}{Lemma}
\newtheorem{definition}{Definition}
\newcommand{\ourmethod}{\texttt{UGCL}}
\newcommand{\tabincell}[2]{\begin{tabular}{@{}#1@{}}#2\end{tabular}}

\newcommand{\yx}[1]{\textcolor{black}{#1}}

\def\BibTeX{{\rm B\kern-.05em{\sc i\kern-.025em b}\kern-.08em
    T\kern-.1667em\lower.7ex\hbox{E}\kern-.125emX}}
\begin{document}

% \title{Unifying Graph Contrastive Learning with Graph Power}
\title{Unifying Graph Contrastive Learning\\ with Flexible Contextual Scopes}

\author{
%Masked for Triple Blind Review
\IEEEauthorblockN{
Yizhen Zheng\IEEEauthorrefmark{2}, Yu Zheng\IEEEauthorrefmark{3}, Xiaofei Zhou\IEEEauthorrefmark{4}, Chen Gong\IEEEauthorrefmark{6}, Vincent CS Lee\IEEEauthorrefmark{2}, Shirui Pan\IEEEauthorrefmark{5}\IEEEauthorrefmark{1}\thanks{\IEEEauthorrefmark{1}Corresponding author.}
}
\IEEEauthorblockA{
\IEEEauthorrefmark{2}%Department of Artificial Intelligence and Data Science, 
Monash University, Australia; %\\
\IEEEauthorrefmark{3}%Department of Computer Science and Information Technology, 
La Trobe University, Australia;\\
\IEEEauthorrefmark{4} %Institute of Information Engineering, 
University of Chinese Academy of Sciences, China; \\
\IEEEauthorrefmark{6}{%School of Computer Science and Engineering, 
Nanjing University of Science and Technology, China; }
\IEEEauthorrefmark{5}%School of Information and Communication Technology, 
Griffith University, Australia \\
{\textit{yizhen.zheng1@monash.edu, yu.zheng@latrobe.edu.au, zhouxiaofei@iie.ac.cn,}}\\
{\textit{chen.gong@njust.edu.cn,
vincent.cs.lee@monash.edu,
s.pan@griffith.edu.au}
}}}

% \author{\IEEEauthorblockN{
% % 1\textsuperscript{st} 
% Anonymous}
% \and
% \and
% \and
% \and
% \and

\maketitle

\begin{abstract}
Graph contrastive learning (GCL) has recently emerged as an effective learning paradigm to alleviate the reliance on labelling information for \yx{graph representation learning}. The core of GCL is to maximise the mutual information between the representation of a node and its \textit{contextual representation} (i.e.,  the corresponding instance with similar semantic information) summarised from the \textit{contextual scope} (e.g., the whole graph or 1-hop neighbourhood). This scheme distils valuable self-supervision signals for GCL training. However, existing GCL methods still suffer from limitations, such as the incapacity or inconvenience in choosing a suitable contextual scope for different datasets and building biased contrastiveness. To address aforementioned problems, we present a simple self-supervised learning method \yx{termed} \underline{U}nifying \underline{G}raph \underline{C}ontrastive \underline{L}earning with Flexible Contextual Scopes (\ourmethod\ for short). Our algorithm builds flexible contextual representations with tunable contextual scopes by controlling the power of an adjacency matrix. Additionally, our method ensures contrastiveness is built within connected components to reduce the bias of contextual representations. Based on representations from both local and contextual scopes,  \ourmethod \ optimises a very simple contrastive loss function for graph representation learning. Essentially, the architecture of \ourmethod\ \textbf{can be considered as a general framework to unify existing GCL methods.} We have conducted intensive experiments and achieved new state-of-the-art performance in six out of eight benchmark datasets compared with self-supervised graph representation learning baselines. Our code has been open sourced\footnote{https://github.com/zyzisastudyreallyhardguy/UGCL}.

% \YXX{The meaning of GRL is not given before.} 

\end{abstract}


\section{Introduction}
Graph neural networks (GNNs) employ a neighbourhood aggregation strategy via iterative message passing to learn low-dimensional node embeddings for permutation-invariant graphs. GNNs have achieved promising results in various graph-based tasks such as node classification \cite{jin2021multi, liu2022towards}, link prediction \cite{zhang2018link}, and graph classification \cite{zhang2018end}. They have been further applied to address various real-world problems such as anomaly detection \cite{liu2021anomaly}, graph similarity computation \cite{jin2022cgmn}, time series forcasting \cite{jin2022multivariate, jin2022neural} and trustworthy systems \cite{zhang2022trustworthy, zhang2021projective}.

\yx{The majority of} GNNs learn node representations \yx{following (semi-)supervised paradigms where supervision signals are provided from manual labels}. However, in the real world, collecting labels is an expensive and labour-intensive process. 
\yx{To address this problem, graph contrastive learning (GCL) methods are produced to alleviate the reliance on labels in graph representation learning \cite{velivckovic2018deep, peng2020graph, hassani2020contrastive, zhu2020deep, jiao2020sub,zheng2021towards, zheng2022rethinking}. } 
% To address this problem, based on the idea of mutual information (MI) maximisation, graph contrastive learning (GCL) methods emerged and have achieved state-of-the-art performance on various downstream tasks  \cite{velivckovic2018deep, peng2020graph, hassani2020contrastive, zhu2020deep, jiao2020sub}. 
The key idea of \yx{GCL} methods is to maximise the \yx{mutual information (MI)} between the representation of a node and its \textit{contextual representation} (i.e., the corresponding node instance with similar semantic information) summarised from the contextual scope (e.g., 1-hop neighbourhood). In particular, aiming to extract global semantic information, global contrasting methods such as DGI \cite{velivckovic2018deep} and MVGRL \cite{hassani2020contrastive} contrast nodes with a readout graph embedding. % MVGRL \cite{hassani2020contrastive} extends DGI by building contrastiveness between two augmented views. 
Focusing on localised information, localised contrasting methods (e.g., GRACE \cite{zhu2020deep}, and GMI \cite{peng2020graph}) maximise MI between a node and its close neighbourhood or augmented counterpart.

Though GCL methods can reduce \yx{the} reliance on labelling information during training, they still share the following deficiencies\yx{:} 1) the establishment of the contextual representation requires a \textit{contextual scope}, \yx{while the} size of this scope is hard to adjust; 2) the aggregated contextual representation is biased in existing GCL methods, as they neglect the independence of connected components. 

% (i.e., contrasting to the whole graph instead of a connected component as shown in Figure \ref{fig:connected component}). 
%the incapacity or inconvenience in tailoring the contextual scope to different datasets, and the bias in building contextual representations neglecting the independence of connected components.
% To obtain the contextual representation of a graph, these methods adopt a readout (i.e., pooling) function to aggregate nodes within the contextual scope. The most common adopted techniques include mean-, max-, min-, and lstm pooling. However, getting contextual representation in this way can suffer from two problems: the difficulty of controling the contextual scope and the overlook of connected components within a graph. 
%Current GCL methods normally utilises pooling mechanism (e.g., mean pooling) to aggregate all nodes embeddings of a graph for contextual representation generation. This approach has no control to the \pan{contextual scope} for contetxual representation generation. Though Subg-con can adjust the contextual scope by tuning the subgraph size, the subgraph creation process requires time-consuming precomputation, especially for large-scale graphs.
%Also, these methods are biased in generating contextual representation since they neglect the independence of connected components in graphs. 

% \begin{figure}[t]
%     \centering
%     \includegraphics[scale = 0.5]{Figures/analogy.png}
%     \caption{An analogy of a graph with three connected components and a set of dog images with three pictures.}
%     \label{fig:analogy}
% \end{figure}

The first limitation arises since most GCL methods only have contextual representations generated from a fixed scope. % (e.g., the whole graph or 1-hop neighbourhood). 
However, with different properties (e.g., type of edges and sparsity), datasets from various domains (e.g., citation networks and social networks) can have different suitable contextual scopes. %For example, edges in citation and social networks are distinct. 
% built based on citation references and social connections. These two edge types are distinct. 
\yx{For example, in} social networks, a faraway neighbour can be semantically unrelated based on the theory of six degrees of separation \cite{milgram1967small}, as all people are no more than six-hop away from each other. However, a remote neighbour can still be similar to a target node in citation networks \yx{since they share the same research field}. In addition, the sparsity of \yx{graphs} can affect the contextual scope since a sparse \yx{graph} may need a larger receptive field to include sufficient informative neighbours. Therefore, selecting a suitable contextual scope for different datasets is necessary. We have provided theoretical justification for why different graphs require different contextual scope in Section \ref{sec theoretical justification}. However, with a fixed scope, existing GCL methods 
% suffer from achieving sub-optimal solutions, which 
cannot \yx{well} exploit supervision signals from the suitable scale for different datasets.

%For the second limitation, some GCL methods (e.g., DGI \cite{velivckovic2018deep} and MVGRL \cite{hassani2020contrastive}) build contrastiveness to a whole graph (i.e., coarsely aggregating all node embeddings) instead of a connected component (as shown in Figure \ref{fig:connected component}). In a graph, connected components are independent because GNNs conduct message-passing via edges. As there is no edge between different connected components, node embedding generation within a connected component will not be affected by other components. Thus, different connected components have their own contextual representations and can diverge. In practice, a graph is not composed by a single connected component but many independent connected components. Therefore, contrasting to a whole graph can be biased as it neglects this independence and mixes up representations from different components, which impairs the model ability in exploring fine-grained information.

For the second limitation, some GCL methods (e.g., DGI \cite{velivckovic2018deep} and MVGRL \cite{hassani2020contrastive}) build contrastiveness between a node and a whole graph, \yx{neglecting} the fact that many graphs in practice are composed of many independent connected components (as shown in Figure~\ref{fig:connected component}). 
\yx{In general}, node embedding generation within a connected component will not be affected by other components. Thus, different connected components \yx{ought to} have their own contextual representations and can be different from each other. \yx{In this case}, contrasting to a whole graph can be biased as it neglects this independence and mixes up representations from different components, which impairs the model ability to explore fine-grained information within each connected component.

\begin{figure}
    \centering
    \includegraphics[scale = 0.55]{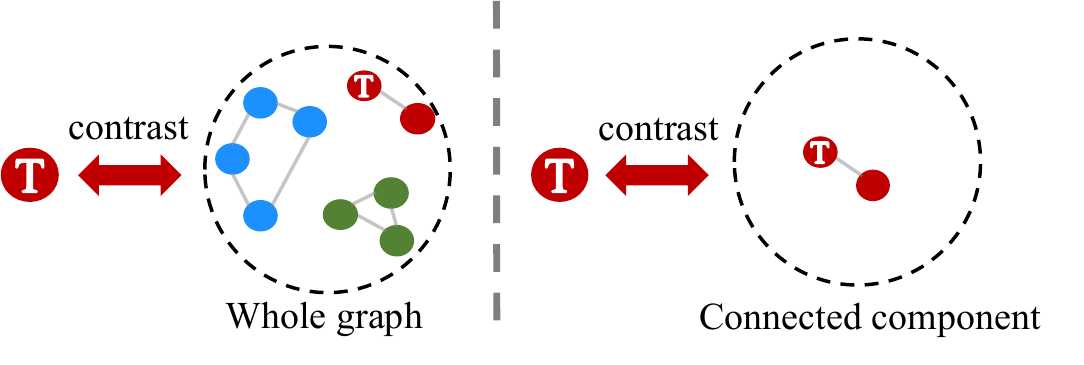}
    \caption{``T'' is the target node, and the red arrow means contrastiveness. The left part shows contrastiveness between ``T'' and a whole graph, while the right part is contrastiveness built within a connected component.}
    \label{fig:connected component}
    \vspace{-4mm}
\end{figure}

To alleviate \yx{the} aforementioned problems in GCL, we propose a new method termed \underline{U}nifying \underline{G}raph \underline{C}ontrastive \underline{L}earning with Flexible Contextual Scopes (namely \ourmethod).
% d, which marries \textit{graph power}~\cite{alon2002graph} - a concept in graph theory, with GCL. %graph power expands the contextual scope based on node proximity.
% can exponentially increase the edges of a graph depending on node proximity and 
% is closely related to the message-passing mechanism of graph convolution. 
% This concept is involved in multiple fields such as information theory (e.g., Shannon Capacity and Witsenhausen's Rate) and Ramsey theory . 
The theme of our algorithm is to establish \textit{contextual representations} by tuning the power of an adjacency matrix, \yx{which flexibly expands the contextual scope based on node proximity.} This mechanism can be regarded as the aggregation which summarises the information embodied in a selected scale. Based on this idea, \ourmethod~can generate contextual representations from a suitable scale for different datasets. Additionally, as graph convolution is only effective on connected components when generating contextual representations, \ourmethod~considers the independence of connected components and reduces the bias of the generated representations by building contrastiveness within a connected component as shown in the right part of Figure~\ref{fig:connected component}. 

% We provide the theoretical guidance for the selection of $n$, which may ease the burden of parameter tuning.

% together with the target node representation are fed into a contrastive learning framework, which tries to maximise the agreement between them. 

% and the smoothing level of representations within connected components. 

% , significantly increase the smoothing level of representations. 
% While over-smoothing is harmful to distinguishing nodes in different classes, with appropriate level of smoothing, it can reduce the bias of node representations in a connected component. 

\vspace{2mm}
\noindent\textbf{Benefits.} \ourmethod\ is conceptually simple, easy to implement, and nicely addresses limitations of common GCL approaches. In particular, it can tune contextual scope easily \yx{and ensure} contrastiveness is conducted within connected components to reduce bias of contextual representations. 
% (as shown in Figure \ref{fig:connected component}). 
{More significantly,} the architecture of \ourmethod~is a general framework that unifies representative GCL approaches, including localised contrasting methods and global contrasting methods.

\section{Preliminary}
\subsection{Problem Definition}
\yx{In this paper, we focus on unsupervised node representation learning problem. }
Given an attributed graph $\mathcal{G = (\textbf{X},\textbf{A})}$, where $\textbf{X} \in \mathbb{R}^{N \times D}$ is feature matrix and $\textbf{A} \in \mathbb{R}^{N \times N}$ denotes the adjacency matrix, we aim to learn a GNN encoder $g(\cdot)$ to generate  node representations without the guidance of labels. Here, $N$ is the number of nodes in $\mathcal{G}$, $D$ is the feature dimension. The output representation, i.e., $\textbf{H} = g(\textbf{X}, \textbf{A}) \in \mathbb{R}^{N \times D'}$, where $D'$ is the hidden embeddings dimension. $\textbf{H}$ can be utilised for various downstream tasks, \yx{such as} node classification \yx{and link prediction}. 

\subsection{Representations Convergence with Raising Power Theorem}
\begin{theorem}[Representations Convergence with Raising Power Theorem]
\label{theorem graph power}
Given an adjacency matrix $\normalfont{\textbf{A}}$, when $n$ increases, multiplying with the $n$-th power of $\normalfont{\textbf{A}}$, node representations $\normalfont{\textbf{H}}$ within a connected component will gradually converge to a shared subspace $\mathcal{M}$ as shown below:
\end{theorem}
\begin{equation}
    \rm d_{\mathcal{M}}(\textbf{A}^{n}\textbf{H}) \leq \lambda d_{\mathcal{M}}(\textbf{A}^{n-1}\textbf{H}),
\end{equation}
\noindent where the definition of subspace $\mathcal{M}$ and the distance of graph representation to $\mathcal{M}$ (i.e., $d_{\mathcal{M}}$), are defined in Appendix.
% denotes the distance of graph representation to $\mathcal{M}$, and 
$\lambda$ is the second largest eigenvalue of $\textbf{A}$. The computation of $\textbf{A}^n$ is formulated as $\textbf{A}^n = \underbrace{\textbf{A}\textbf{A} \cdots \textbf{A}}_{n}$.

This theorem shows that as $n$ increases, node representations will converge to a subspace $\mathcal{M}$ as $\lambda$ is guaranteed to be smaller than 1 (as shown in Lemma \ref{lemma norm} in Appendix). Thus, with sufficiently large $n$-th power for $\textbf{A}$, the generated contextual node representation can summarise all node embeddings within a connected components. This is because it encodes information of the whole graph. The detailed proof of this theorem is presented in Appendix. 

\subsection{Contextual Homophily Rate \& Graph Sparsity}
Here, we define the contextual homophily rate $\mathcal{P}_n(i)$ to evaluate the homophily rate (i.e., rate of neighbours sharing the same label as the target node) in the contextual scope and graph sparsity $T_{\mathcal{G}}$.
\begin{definition}[Contextual Hompohily Rate]
The contextual homophily rate of a node $i$ is $\mathcal{P}_n$(i), where $n$ means its contextual scope is based on $n$-th power of $\textbf{A}$ (i.e., $n$-hop neighbourhood):
% . Given a node $i$, its contextual homophily rate on $n$-th power of $\textbf{A}$ is
\begin{equation}
    \mathcal{P}_n(i) = |\frac{j \in \mathcal{N}_n(i) \wedge y_i = y_j}{j \in \mathcal{N}_n(i)}|,
\end{equation}
where $y$ is label for a node, and $\mathcal{N}_n(i)$ is the neighbourhood for node $i$ with $n$-th power of $\textbf{A}$, i.e., n-hop neighbourhood. 
\end{definition}

\begin{definition}[Graph Sparsity]
The sparsity $T_{\mathcal{G}}$ of a given graph $\mathcal{G}$ is:
\begin{equation}
T_{\mathcal{G}} = \frac{E}{N \times N} \approx \frac{N \times d}{N \times N} ,
\end{equation}
where $d$ is the average degree of nodes in $\mathcal{G}$, while $E$ and $N$ represent the number of edges and nodes in $\mathcal{G}$ respectively.
\label{sparsity}
\end{definition}

\section{Method}
% We propose a novel method, namely \ourmethod, to learn node representations via graph power in a self-supervised fashion. 

\yx{In this section, we introduce the proposed \ourmethod~which learns node representations via the power adjustment of $\textbf{A}$ in a self-supervised fashion.} 
The overall architecture of our method is illustrated in Figure~\ref{fig:overall}. To train our model, we first create two views: patch- and contextual view, where the latter view is generated with the $n$-th power of $\textbf{A}$. Then, we construct a cross-view contrastiveness in a pair-wise contextual relationship between these two views. The following sections illustrate the details of view establishment and the cross-view contrastiveness of \ourmethod. 
% and the scalability extension of our method for handling large-scale graphs.

\begin{figure*}[h]
\centering
\includegraphics[scale = 0.45]{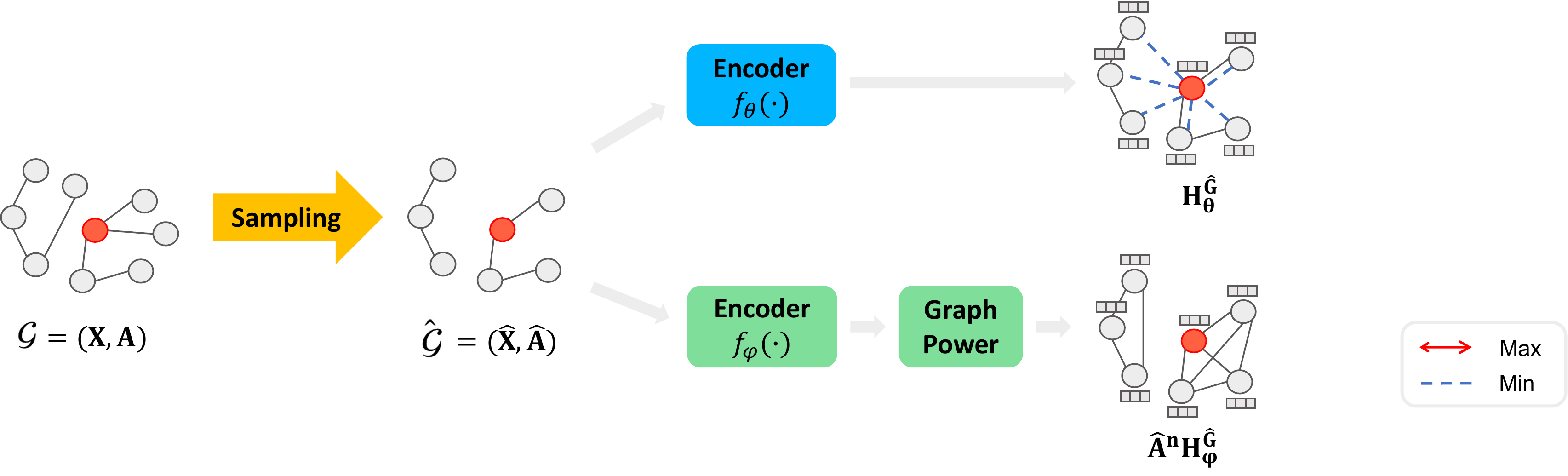}
\caption{The overall architecture of \ourmethod. Firstly, we conduct subsampling on $\mathcal{G} = (\textbf{X},\textbf{A})$ to construct its subgraph $\hat{\mathcal{G}}$, whose size is predefined by a parameter $S$, 
% i.e., $\hat{\textbf{X}} \in \mathbb{R}^{S \times D}$ and $\hat{\textbf{A}} \in \mathbb{R}^{S \times S}$. 
Then, $\hat{\mathcal{G}}$ is fed into two GNN-based encoders, the primary encoder $f_{\theta}(\cdot)$ and the auxiliary encoder $f_{\varphi}(\cdot)$, to generates node representations $\textbf{H}^{\hat{\mathcal{G}}}_{\theta}$, and % the auxiliary encoder $f_{\varphi}$ yeilds
$\textbf{H}^{\hat{\mathcal{G}}}_{\varphi}$, respectively. After that, $\textbf{H}^{\hat{\mathcal{G}}}_{\varphi}$ is multiplied with $n$-th power of $\hat{\textbf{A}}$ to obtain contextual representations. Finally, a contrastive learning scheme is deployed, where the red solid line indicates MI maximisation, while multiple blue dash lines represent the opposite operation.}
\label{fig:overall}
\vspace{-0.6cm}
\end{figure*}

% Our proposed method includes three parts: 1) View Establishment; 2) Cross-View Contrastive Learning and 3) Scalability Extension. In View Establishment, we generate two views for contrastive learning. We discuss the contrastive learning process in Cross-View Contrastive Learning, and extend the proposed approach to large scale graphs.

\subsection{View Establishment}
\label{view establisment}
In our method, the view establishment process generates two views (i.e., patch view and contextual view), based on which a cross-view contrastive learning scheme is employed to compute contrastive loss. As shown in Figure \ref{fig:overall}, a subgraph $\hat{\mathcal{G}} = (\hat{\textbf{X}}, \hat{\textbf{A}})$ is sampled from $\mathcal{G}$ 
% (i.e., subsampling augmentation for scalability extension)
and fed into two GNN encoders, the main encoder $f_{\theta}(\cdot)$ and the auxiliary encoder $f_{\varphi}(\cdot)$, to get two variants of node representations $\textbf{H}^{\hat{\mathcal{G}}}_{\theta}$ and $\textbf{H}^{\hat{\mathcal{G}}}_{\varphi}$ for $\hat{\mathcal{G}}$. In our experiment, we adopt a one-layer GCN as the GNN encoder. The details of the subsampling process is presented in the subsection below.
Here, we consider $\textbf{H}^{\hat{\mathcal{G}}}_{\theta}$ as the patch view representation. To obtain the contextual view representation, we first compute the $n$-th power of $\hat{\textbf{A}}$ and multiply the output with $\textbf{H}^{\hat{\mathcal{G}}}_{\varphi}$. This process can be formulated as follows:
\vspace{-2mm}
\begin{equation}\label{eq:graph_power}
    \tilde{\textbf{H}}^{\hat{G}}_{\varphi} = \hat{\textbf{A}}^n \textbf{H}^{\hat{G}}_{\varphi},
\end{equation}
\vspace{-6mm}

\noindent where $\tilde{\textbf{H}}^{\hat{\mathcal{G}}}_{\varphi}$ is the contextual view representation, and $n$ is a tunable parameter. It is worth noting that the computation of $\hat{\textbf{A}}^n$ can be easily relieved with sub-sampling and matrix multiplication decomposition. The computation time of this power mechanism is less than or around 1 millisecond for five datasets of various sizes (as shown in Table \ref{graph power}). In our proposed method, $n$ is a key parameter to control the contextual scope of contextual representation $\tilde{\textbf{H}}^{\hat{\mathcal{G}}}_{\varphi}$. As the $n$-th power of $\textbf{A}$ gives the number of paths of length $n$ between two nodes, two vertices are adjacent if the distance between these two vertices is less than or equal to $n$ \cite{balakrishnan2012textbook}.
%According to the graph power theorem, the $n$-th power of $\textbf{A}$ has the same set of nodes as $\mathcal{G}$, where  . 
Therefore, by multiplying $\hat{\textbf{A}}^n$ with $\textbf{H}^{\hat{\mathcal{G}}}_{\varphi}$, the contextual scope of contextual representations can be extended to $n$-hop neighbourhood. 

\noindent\textbf{Subsampling}. 
% For subsampling, we adopt a very simple yet effective subsampling approach. 
\yx{We adopt a very simple yet effective subsampling process for data augmentation. }
Specifically, we randomly pick a preset number of nodes and their edges to form a subgraph for training in each training epoch. The advantages of this approach are two folds: preserving essential properties (i.e., the sparsity $T_{\mathcal{G}}$ and the contextual homophily rate of $\textbf{A}^1$, $\mathcal{P}_1(i)$) of the given graph $\mathcal{G}$, and building diversified subgraphs with trivial computation and sufficient randomness.

To prove the first advantage, we propose the following proposition: 
\begin{proposition}
Given a Graph $\mathcal{G}$ with $d$ average node degree and $\mathcal{P}_1(i)$ homophily rate for the first power of $\textbf{A}$, its sampled graph $\hat{\mathcal{G}}$ still have similar sparsity $T_{\hat{\mathcal{G}}}$ and $\mathcal{P}_1(i)$ as $\mathcal{G}$.

% The contextual homophily rate $\mathcal{P}_n$ would not change after subsampling.
\end{proposition}

\begin{proof}
    As the sparsity of $\mathcal{G}$, $T_{\mathcal{G}}$, equals to $\frac{E}{N \times N}$ and $E \approx N \times d$, we can derive that $T_{\mathcal{G}} \approx \frac{N \times d}{N \times N} = \frac{d}{N}$. After sampling, we can obtain a subgraph $\hat{\mathcal{G}}$, which has $S$ nodes. For each node, it would have $S \times d \times \frac{S}{N}$ neighbours, i.e., edges. Thus, the sparsity of the subgraph $\hat{\mathcal{G}}$, $T_{\hat{\mathcal{G}}} \approx \frac{S \times d \times \frac{S}{N}}{S \times S} = \frac{d}{N}$. As $T_{\mathcal{G}}$ is approximately equal to $T_{\hat{\mathcal{G}}}$, we show that the sparsity of $\mathcal{G}$ and $\hat{\mathcal{G}}$ is similar. In addition, as edges in $\hat{\mathcal{G}}$ come from the original graph $\mathcal{G}$, they still connect approximately the same ratio ($\mathcal{P}_1(i)$) of homophilic neighbours (i.e., nodes sharing the same label as $i$). Here, we prove the above proposition.
    % Based on Equation \ref{low bound p}, we can see the contextual homophily rate of $n$-th graph power is affected by $T$ and $\mathcal{P}_1$. As these two variables unchange in $\hat{\mathcal{G}}$, $\mathcal{P}_n$ of $\hat{\mathcal{G}}$ should be the same as $\mathcal{G}$ and 
\end{proof}

The second advantage alleviates the reliance on the fixed graph during model training. As we sample a subgraph in each training epoch, the sampled subgraph is changing instead of in the static state. As a result, the model has to be versatile to handle the contrastiveness built with changing topology. This can be regarded as an augmentation to increase the difficulty of the self-supervised pre-text tasks, which may improve the model performance. We have conducted an ablation study in Section \ref{ablation study section} to show its effectiveness.

\subsection{Cross-View Contrastive Learning}
In \ourmethod, the cross-view contrastive learning scheme consists of two contrastive paths, which are the patch-view contrastive path and the cross-view contrastive path. In the same view, the first path distinguishes an anchor node embedding from other node embeddings, which are considered negative samples. The latter path simply maximises the cosine similarity between a patch view representation and its corresponding contextual representation (i.e., positive samples). By combining both paths,
% positive and negative samples deriving from these two paths, 
we can form the contrastive learning objective.

\yx{As shown in Figure \ref{fig:overall}}, after processing $\hat{\mathcal{G}}$ into the primary GNN encoder $f_{\theta}(\cdot)$, we can get the patch view representation $\textbf{H}^{\hat{\mathcal{G}}}_{\theta}$. Within the patch view, giving the set of nodes $V$ in $\hat{\mathcal{G}}$ and an anchor node $v \in V$, we define that all nodes except for the anchor node as negative samples to regularise the contrastive loss via MI minimisation. 
% Specifically, we minimise the cosine similarity between anchor nodes and these samples (blue dash line in Figure \ref{fig:overall}). 
In addition, we define an anchor node representation in the patch view $h_v \in \textbf{H}^{\hat{\mathcal{G}}}_{\theta}$ and contextual view $\tilde{h}_v \in \tilde{\textbf{H}}^{\hat{\mathcal{G}}}_{\varphi}$ as the positive pair (red line in Figure \ref{fig:overall}). By discriminating representations in the positive pair, our model can distil self-supervision signals from the chosen contextual scope. The contrastive loss function can be formulated as follows:
\vspace{-1mm}
\begin{equation}\label{eq:Loss}
     \mathcal{L} = -\frac{1}{S}\sum^S_{v \in V} \log \frac{e^{\cos(\textit{h}_v, \tilde{\textit{h}}_v)}}{\sum^N_{u \in V;u\neq v}e^{\cos(\textit{h}_v, \textit{h}_u)}},
\end{equation}

\noindent where $\cos(\cdot)$ is the cosine similarity function, $S$ represents the number of nodes in the sampled graph, $h_v$ and $\tilde{h}_v$ denote the anchor node patch- and contextual representation respectively.

% Different from GRACE \cite{zhu2020deep} which adopts a symmetric objective and constructs negative samples in cross-view contrastiveness, we adopt an asymmetric objective and have no cross-view negative samples. This is because when $n$ is large for graph power, the contextual representations can be oversmoothed (i.e., nodes representations in the same connected component become similar). In this case, if we build cross-view negative samples, we are building both positive and negative relationships towards the set of nodes within the same connected component, which is contradictory.

% With a symmetric loss, we need to minimise the similarity between a node and other nodes within the same connected component in the contextual view. This is not ideal as with over-smoothing, nodes representations in the same connected component should be very similar or identical. Similarly, with cross-view negative samples, 

\subsection{Model Training}
To train our model end-to-end, we leverage the loss $\mathcal{L}$ defined in Equation (\ref{eq:Loss}). The training objective is to minimise $\mathcal{L}$ during the optimisation. To obtain the output embeddings for downstream tasks, we first generate $\textbf{H}^\mathcal{G}_{\theta}$ with the trained GNN encoder $g_{\theta}$ and $\tilde{\textbf{H}}^\mathcal{G}_{\theta}$ by multiplying $\textbf{H}^\mathcal{G}_{\theta}$ with $\textbf{A}^n$. Finally, we aggregate these two representations: $\textbf{H} = \textbf{H}^\mathcal{G}_{\theta} + \tilde{\textbf{H}}^\mathcal{G}_{\theta}$ to get the final representations. 
% The detailed algorithm of our proposed method is provided in Appendix B.

\section{Unifying Representative GCL Methods}
% GCL learns self-supervised representations by injecting contrastiveness between patch- and contextual view. 
\yx{To learn node representations in a self-supervised manner, GCL methods usually inject contrastiveness between patch view and contextual view \cite{liu2021graph}. }  
 In Figure \ref{fig: unify}, we present the general architecture of \ourmethod, which can be considered as a unified framework of GCL methods. From the framework, GCL generally follows three steps: augmentation, graph encoding and contrasting. Specifically, the optional first step is applying augmentation to the original graph $\mathcal{G} = (\textbf{X}, \textbf{A})$ to create semantically similar graph instances $\hat{\mathcal{G}}$. 
In \ourmethod, we regard subsampling as augmentation. Then, the graph encoder $g(\cdot)$ generates node representations $\textbf{H}^{\hat{\mathcal{G}}}_g = g(\hat{\mathcal{G}})$. To train $g(\cdot)$, contrastiveness is built between node representations and their corresponding contextual representations to acquire self-supervision signals via MI maximisation:%. Under this framework, 
\begin{equation}
    g^{*}(\cdot) = \mathop{\arg\max}\limits_{g} \textbf{MI}(\textbf{H}^{\hat{\mathcal{G}}}_{g}, \textbf{A}^n\textbf{H}^{\hat{\mathcal{G}}}_{g}),
    \vspace{-2mm}
\end{equation}
% \begin{equation}
%     g^{*}= \operatornamewithlimits{argmax}_{g} \textbf{MI}(\textbf{H}^{\hat{\mathcal{G}}}_{g}, \textbf{A}^n\textbf{H}^{\hat{\mathcal{G}}}_{g}),
% \end{equation}
where $g^{*}(\cdot)$ is the trained encoder, $\textbf{MI}(\cdot)$ is a mutual information neural estimator \cite{belghazi2018mutual} consisting of discriminative network (i.e., bilinear transformation or cosine similarity) and contrastive loss, $\textbf{H}^{\hat{\mathcal{G}}}_{g}$ and $\textbf{A}^n\textbf{H}^{\hat{\mathcal{G}}}_{g}$ represent node- and contextual representations respectively. In the following sections, we interprets four representative GCL methods of two categories (i.e., localised- and global contrasting methods) with the proposed unified framework.

\begin{figure}
    \centering
   \includegraphics[height=0.25\textwidth]{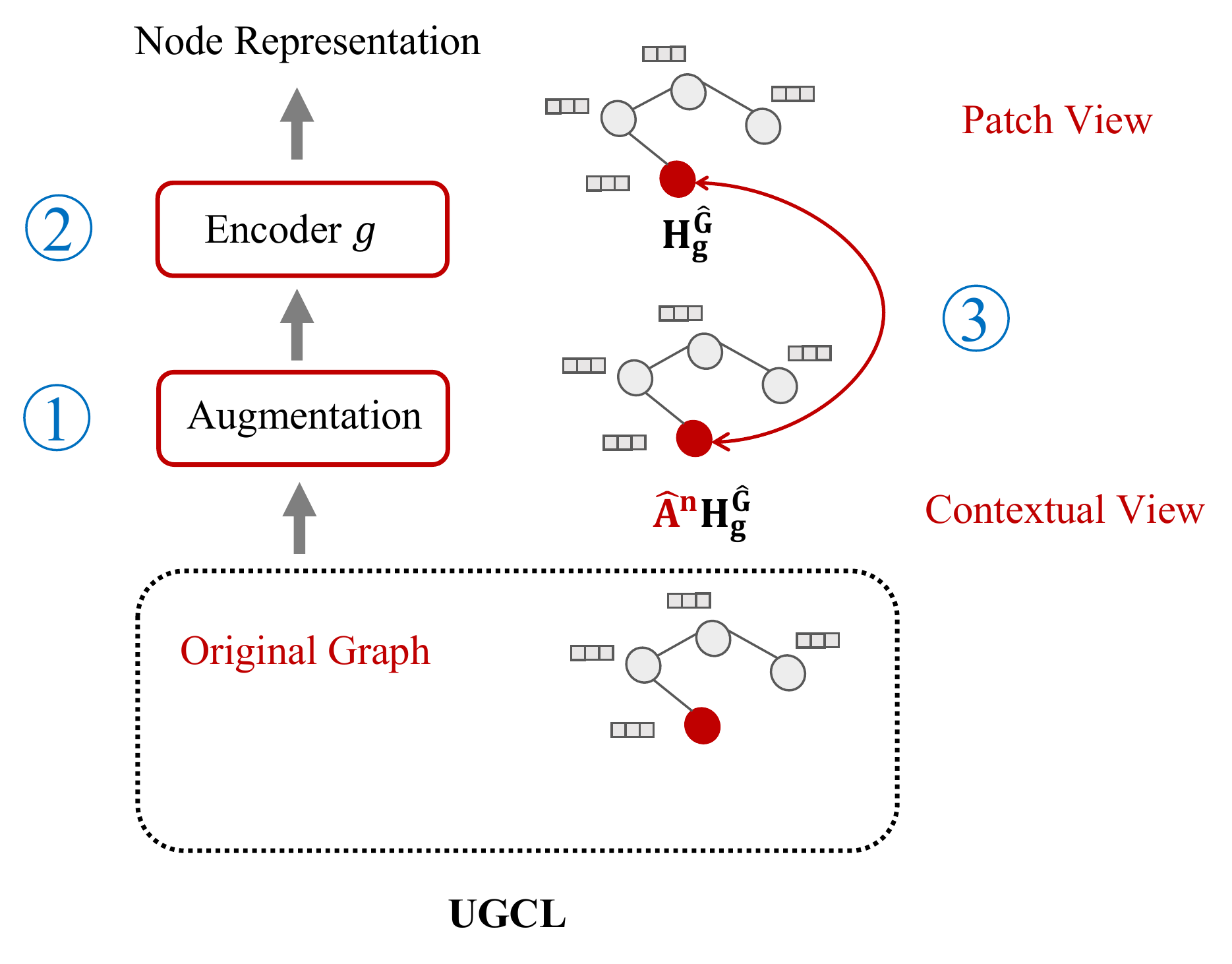}
    \caption{A unified framework for GCL methods. The red line indicates contrastiveness.}
    \label{fig: unify}
    \vspace{-6mm}
\end{figure}

\begin{figure*}[t]
 \centering
 \vspace{-0.4cm}
 \hspace{-0.2cm}
 \subfigure[Localised contrasting methods]{
  \includegraphics[height=0.25\textwidth]{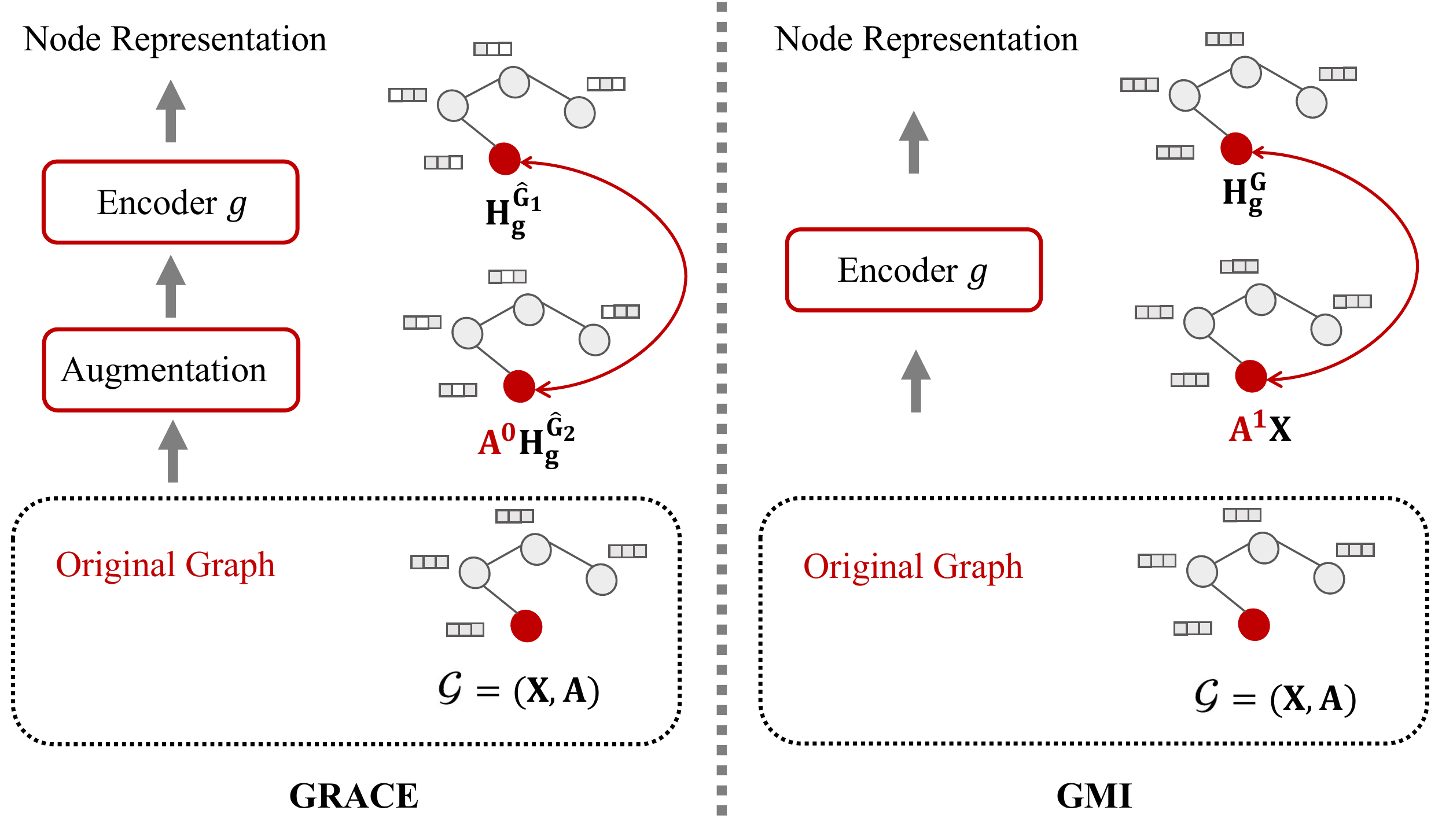}
  \label{subfig:local}
 }
 \subfigure[Global contrasting methods]{
  \includegraphics[height=0.25\textwidth]{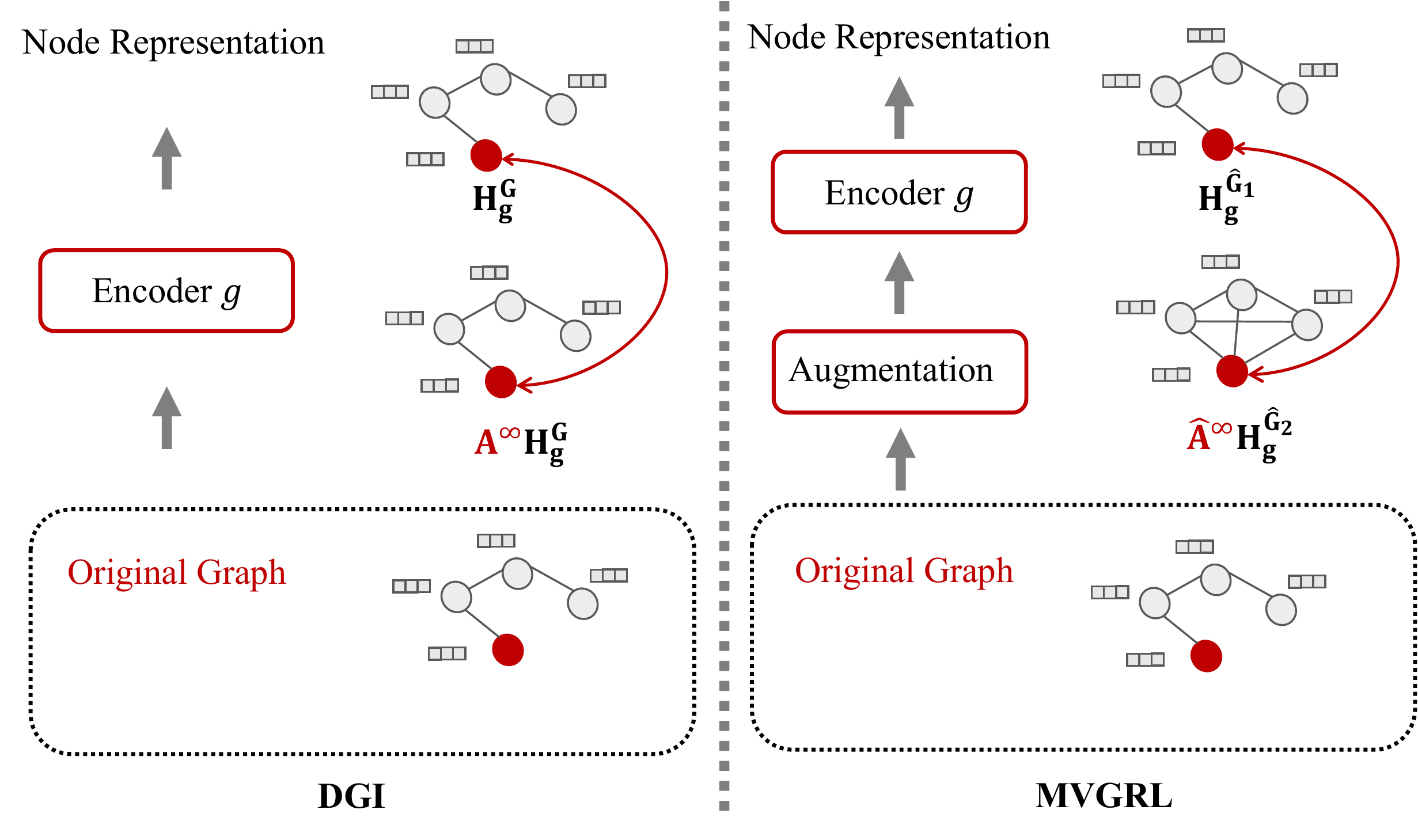}
  \label{subfig:global}
 }
 \vspace{-0.2cm}
 \caption{(a) Localised contrasting methods interpreted with the architecture of \ourmethod. (b) Global contrasting methods interpreted with the architecture of \ourmethod.}
 \vspace{-0.55cm}
 \label{fig:sim_intro}
\end{figure*}

\subsection{Connections to Localised Contrasting Methods}
Localised contrasting methods form pretext tasks by pulling the representation of a node closer to its augmented counterpart or close neighbours to distil the localised contextual information \cite{zhu2020deep, peng2020graph, jiao2020sub}. These methods can be considered as \ourmethod\ with a tiny $n$ for the power of  $\textbf{A}$ when building the contextual view. To illustrate this point, we present GRACE \cite{zhu2020deep} and GMI \cite{peng2020graph} (i.e., two typical localised methods) with the unified framework in Figure \ref{subfig:local}. 

% \begin{figure}[ht]
%     \centering
%     \includegraphics[scale = 0.38]{Figures/gcl_gp_general.pdf}
%     \caption{The architecture of GCL-GP, which can be considered as a unified framework for existing GCL methods. It consists of three components: augmentation, encoder and contrastiveness. The red line represents the contrasting relationship.}
%     \label{fig:general_framework}
%     \vspace{-4mm}
% \end{figure}

In particular, GRACE can be considered as contrasting a node to a contextual representation with 0-th power, i.e., $\textbf{A}^0=\textbf{I}$. It employs two different graph augmentation methods to generate two augmented views $\hat{\mathcal{G}}_1$ and  $\hat{\mathcal{G}}_2$, where it pulls the representations of the same node in these two views closer. This node-to-node comparison strategy allows GRACE to extract the most fine-grained information from the contextual scope with only one node (i.e., a node's augmented counterpart). This scheme is equal to applying $\textbf{A}^0$ to the generated contextual representation, as we only care about the semantic information embodied within a node itself:
 \vspace{-2mm}

\begin{equation}
g^{*}(\cdot)=\underset{g}{\rm arg\ max} \operatorname{\textbf{MI}}\left(\textbf{H}_{g}^{\hat{\mathcal{G}}_{1}}, \textbf{A}^{0} \textbf{H}_{g}^{\hat{\mathcal{G}}_2}\right).
\label{eq GRACE}
\vspace{-2mm}
\end{equation}

Different from GRACE, GMI extends the contextual scope to 1-hop neighbourhood. It maximises the MI between a node and the raw features of its 1-hop neighbours. 
% As GMI focuses on extracting supervision signals from closet neighbours, it 
This is similar to contrasting the raw features $\textbf{X}$ with 1-th power of $\textbf{A}$, which aggregates a node 1-hop neighbourhood:
 \vspace{-2mm}

% \noindent%
% \begin{tabular}{@{}*{2}{m{\dimexpr0.5\linewidth-2\tabcolsep\relax}}@{}}

\begin{equation}
g^{*}(\cdot)= \underset{g}{{\rm arg\ max}}\operatorname{\textbf{MI}}\left(\textbf{H}_{g}^{\mathcal{G}}, \textbf{A}^{1} \textbf{X}\right).
 \label{eq GMI}
 \vspace{-2mm}
\end{equation}
% \end{tabular}

% \begin{equation}
%     g^{*}=\underset{g}{\rm argmax} \operatorname{\textbf{MI}}\left(\textbf{H}_{g}^{\hat{\mathcal{G}}_{1}}, \textbf{A}^{0} \textbf{H}_{g}^{\hat{\mathcal{G}}_2}\right);g^{*}= \underset{g}{{\rm argmax}}\operatorname{\textbf{MI}}\left(\textbf{H}_{g}^{\mathcal{G}}, \textbf{A}^{1} \textbf{X}\right).
% \end{equation}

% \begin{equation}
%     g^{*}= \underset{g}{{\rm argmax}}\operatorname{\textbf{MI}}\left(\textbf{H}_{g}^{\mathcal{G}}, \textbf{A}^{1} \textbf{X}\right)
% \end{equation}

This contextual representation can provide the very-local information of a node for contrastiveness. However, focusing only on the close neighbourhood, these localised methods neglect useful information from a broader receptive field.

% Similar to GRACE \cite{zhu2020deep} and GMI \cite{peng2020graph}.
% SubG-Con \cite{jiao2020sub} also utilises localised information to build a contrastive learning scheme. It creates small-scale subgraphs consisting of a node's most important neighbours. It contrasts the node to its corresponding summarised subgraph representation, which is similar to the contextual representation with a tiny n-th graph power.

% \begin{figure}[ht]
%     \centering
%     \includegraphics[scale = 0.34]{Figures/localized.pdf}
%     \caption{Localised contrasting methods interpreted with the architecture of \ourmethod.}
%     \label{fig:local}
%     \vspace{-2mm}
% \end{figure}

% \begin{figure}[ht]
%     \centering
%     \includegraphics[scale = 0.34]{Figures/global.pdf}
%     \caption{Global contrasting methods interpreted with the architecture of \ourmethod.}
%     \label{fig:global}
    
% \end{figure}

\subsection{Connections to Global Contrasting Methods}
Global contrasting methods place discrimination between a node and a graph-level embedding, summarising all node representations in a graph \cite{hassani2020contrastive, velivckovic2018deep}. In a special case (i.e., the input graph only has one connected component), these approaches are similar to \ourmethod\ with infinite power for $\textbf{A}$. The interpretation of DGI and MVGRL in the architecture of \ourmethod\ is presented in Figure \ref{subfig:global}. 
% As the first method to apply a node-graph contrasting scheme, 
DGI aims to extract global semantic information from graph-level embedding. Specifically, they create the graph-level embedding by coarsely averaging all node embeddings in a graph. % via mean pooling
MVGRL 
% adopts the same approach as DGI and 
extends DGI with additional augmentation, which creates two augmented views $\hat{\mathcal{G}}_1$ and $\hat{\mathcal{G}}_2$.

Similarly, we can also apply the infinite power of $\textbf{A}$ to generate graph-level contextual representations. According to Theorem \ref{theorem graph power}, with the infinite power of $\textbf{A}$, node representations within a connected component become similar and converge to a certain point or shared subspace. This is because nodes iteratively receive messages from all other nodes within the same connected component, which smooths the difference between these nodes. Thus, the oversmoothed representations can be regarded as the graph-level embedding in the aforementioned special case (i.e., the input graph only has one connected component). These global contrasting methods can be approximated by applying infinite power of $\textbf{A}$ to the contextual view in \ourmethod:

\noindent%
% \begin{tabular}{@{}*{2}{m{\dimexpr0.5\linewidth-2\tabcolsep\relax}}@{}}
  \begin{equation}
g^{*}(\cdot)=\underset{g}{\rm arg\ max} \operatorname{\textbf{MI}}(\textbf{H}_{g}^{\mathcal{G}}, \textbf{A}^{\infty} \textbf{H}^{\mathcal{G}}_g),
\label{eq DGI}
  \end{equation}
  \begin{equation}
 g^{*}(\cdot)=\underset{g}{\rm arg\ max} \operatorname{\textbf{MI}}(\textbf{H}_{g}^{\hat{\mathcal{G}}_1}, \textbf{A}^{\infty} \textbf{H}_g^{\hat{{\mathcal{G}}_2}}),
 \label{eq MVGRL}
 \vspace{-2mm}
  \end{equation}
  where Equation (\ref{eq DGI}) and Equation (\ref{eq MVGRL}) are formulas for DGI and MVGRL in the unified framework. Though these global contrasting methods can extract supervision signals from the global view, with the largest contextual scope, it is inevitable to include irrelevant noise and impede the model training. 

% \begin{wraptable}{r}{0.45\textwidth}
%     \vspace{-0.6cm}
%     \begin{minipage}{0.45\textwidth}
% \caption{The difference between five GCL methods in two key properties (i.e., augmentation and contextual scope).} 
% 	\footnotesize
% \vspace{0.1cm}
% \resizebox{1\columnwidth}{!}{
%     \begin{tabularx}{1.35\linewidth}{lp{3.2cm}<{\centering} p{3.2cm}<{\centering}}
%     \toprule
%     {\textbf{Property}} & \textbf{Augmentation} &     \textbf{Contextual Scope} \\
%     \midrule
%     \textbf{GRACE}  &          Yes &                 Node \\
%     \textbf{GMI}    &           No &  1-hop Neighbourhood \\
%     \textbf{DGI}    &           No &          Graph-level \\
%   \textbf{MVGRL}  &          Yes &          Graph-level \\
%   \midrule
%     \textbf{GCL-GP} &          Yes &             Flexible \\
%     \bottomrule
%     \end{tabularx}
%     \label{tab:difference}}
%     \end{minipage}
%     \vspace{-0.5cm}
% \end{wraptable}
\subsection{Theoretical Justification for Flexible Contextual Scope}
\label{sec theoretical justification}

% We show the difference of two key properties (i.e., augmentation and contextual scope) between our method and four representative GCL approaches in Table \ref{tab:difference}. Compared with these GCL methods, \ourmethod~ can have a flexible contextual scope, which allows us to choose the most suitable scale for different datasets. 

Here, we provide theoretical justification for why different graphs require different contextual scope from the perspective of homophily dominance \cite{zhu2020beyond}. Moreover, based on this theoretical analysis, we can guide the selection of $n$-th power for $\textbf{A}$ in model training. 

Existing GNNs assume homophily in prior graphs to be effective. To obtain effective representations, nodes in the contextual scope are expected to be homophily dominant\cite{zhu2020beyond}, which can be achieved when $\mathcal{P}_n(i)$ is at least larger than any $\{|\frac{j \in \mathcal{N}_n(i) \wedge y_j \in c}{j \in \mathcal{N}_n(i)}|; c \in C\}$, where $C$ is number of classes for the graph. Here we consider $i$ as a node with the average degree of $\mathcal{G}$, $d$, and assume $\mathcal{P}_1(i)$ equals to the homophily rate of a graph, i.e., the proportion of edges connecting same class nodes. To hold homophily dominance, $\mathcal{P}_n(i)$ needs to be as large as possible. The following lemma shows the lower bound of $\mathcal{P}_n(i)$ is determined by $\mathcal{P}_1(i)$ and $d$:

\begin{lemma}
\textit{Given $\mathcal{G}$ with an average degree of $d \geq 1$ for each node and contextual homophily rate with 1-th power $\mathcal{P}_1(i)$. With $n$-th power of $\textbf{A}$, the lower bound of $\mathcal{P}_n(i)$ is:}
\begin{equation}
    \mathcal{P}_n(i) > \frac{(d-1)\mathcal{P}_1(i)(d^n\mathcal{P}_1^n(i)-1)}{(d^n-1)(d\mathcal{P}_1(i)-1)}.
\label{low bound p}
 \vspace{-6mm}
\end{equation}
\label{lemma p}
\end{lemma}

% Then we define the sparsity $T$ for $\mathcal{G}$:

Based on Lemma \ref{lemma p} and Definition \ref{sparsity}, we can derive \yx{four} properties:

\begin{property}
\label{theorem n}
Given a graph $\mathcal{G}$, when $n$ increases, the lower bound of $\mathcal{P}_n(i)$ drops.
\end{property}

\begin{property}
\label{theorem p}
Given a graph $\mathcal{G}$, when $\mathcal{P}_1(i)$ increases, the lower bound of $\mathcal{P}_n(i)$ increases.
\end{property}

\begin{property}
\label{theorem d}
Given a graph $\mathcal{G}$, when $d$ increases, the lower bound of $\mathcal{P}_n(i)$ drops.
\end{property}

\begin{property}
\label{theorem t}
Given a graph $\mathcal{G}$, when $T_{\mathcal{G}}$ increases, the lower bound of $\mathcal{P}_n(i)$ drops.
\end{property}

Property \ref{theorem p}, Property \ref{theorem d} and Property \ref{theorem t} indicate the choice of $n$ is closely related to two essential properties of graphs, i.e., sparsity and homophily. Specifically, when a graph is dense (i.e., the sparsity $T_{\mathcal{G}}$ is large), a large $n$ can easily break the homophily dominance and degrade model performance. In addition, under the same level of sparsity, a low $n$ is preferable for a low $\mathcal{P}_1(i)$, i.e., homophily rate. These findings are consistent with our experiment results shown in Section \ref{n-th graph power}. The detailed proof of Lemma \ref{lemma p}, Properties \ref{theorem n}-\ref{theorem t} are provided as follows:

\begin{proof}
Give a graph $\mathcal{G}$ with an average degree of $d$ for each node and homophily rate with 1-th power of $\textbf{A}$, $\mathcal{P}_1(i)$. $d$, $\mathcal{P}_1(i)$ and $n$ are all positive numbers. With $d$, the total number of nodes in the neighbourhood for $n$-th power of $\textbf{A}$ is equal to $d + d^2 + \cdots + d^n$. For the number of homophilic nodes in the neighbourhood, we know it is at least $d\mathcal{P}_1(i)$ for the first-hop neighbourhood, $d^2\mathcal{P}_1(i)^2$ for the second-hop neighbourhood, and so on. Thus, the lower bound of this number is equal to $d\mathcal{P}_1(i) + d^2\mathcal{P}_1(i)^2 + \cdots + d^n\mathcal{P}_1(i)^n$. Here, we can formulate the lower bound of $\mathcal{P}_n(i)$ as:
\begin{equation}
\label{eq lemma 1}
\begin{aligned}
    \mathcal{P}_n(i) &> \frac{\sum _{k=1}^n d^k \mathcal{P}_1(i)^k}{\sum _{k=1}^n d^k} \\
    &= \frac{(d-1)\mathcal{P}_1  \left(d^n \mathcal{P}_1(i)^n-1\right)}{\left(d^n-1\right) (d \mathcal{P}_1(i)-1)}.
\end{aligned}
\end{equation}

From the above formula, we first prove Property \ref{theorem n}. The gap between denominator and numerator of the lower bound of $\mathcal{P}_n(i)$ would become larger when $n$ increases. This is because when $n$ increases by 1, the numerator would increase $d^n\mathcal{P}_1(i)^n$, while the denominator would increase $d^n$. It is easy to observe the gap between $d^n\mathcal{P}_1(i)^n$ and $d^n$ would become larger and larger when $n$ increases as $0 < \mathcal{P}_1(i) < 1$. Thus, we prove that when $n$ increases, the lower bound of $\mathcal{P}_n(i)$ drops.

To prove Property \ref{theorem p}, we present the partial derivative with respect to $\mathcal{P}_1(i)$ for the numerator of the lower bound of $\mathcal{P}_n(i)$:
\begin{equation}
\begin{aligned}
    f^{\mathcal{P}_n(i)}_{num} &= \sum_{k=1}^n d^k \mathcal{P}_1(i)^k, \\
    \frac{\partial f^{\mathcal{P}_n(i)}_{num}}{\partial \mathcal{P}_1(i)} &= d + 2d^2\mathcal{P}_1(i) + 3d^3\mathcal{P}_1(i)^2 + \cdots + nd^n\mathcal{P}_1(i)^{n-1}, 
\end{aligned}  
\end{equation}
where $f^{\mathcal{P}_n(i)}_{num}$ is the numerator for $\mathcal{P}_n(i)$. As $d \geq 1$ and $\mathcal{P}_1 > 0$, $\frac{\partial f^{\mathcal{P}_n(i)}_{num}}{\partial \mathcal{P}_1(i)}$ is a positive number. In addition, changing $\mathcal{P}_1(i)$ will not affect the denominator of the lower bound of $\mathcal{P}_n(i)$. Thus, we can derive that increasing $\mathcal{P}_1(i)$ will lead to the monotonical increase for the lower bound of $\mathcal{P}_n(i)$ and prove Property \ref{theorem p}.

From Equation \ref{eq lemma 1}, we can also prove Property \ref{theorem d} as follows:
\begin{equation}
    \begin{aligned}
    \mathcal{P}_n(i) &> \frac{\sum _{k=1}^n d^k \mathcal{P}_1(i)^k}{\sum _{k=1}^n d^k} \\
    &= \sum^n_{k=1}(\frac{d^k}{\sum^n_{k=1}d^k})\mathcal{P}_1(i)^k,
    \end{aligned}
\end{equation}
here we define $\frac{d^k}{\sum^n_{k=1}d^k}$ as $M$, then we can obtain the derivative of $M$:
\begin{equation}
    \begin{aligned}
    \frac{dM}{dd} &= \sum^{n}_{k=1}\frac{d^{k-1}(\sum^n_{k=1}d^k)-d^k(\sum^n_{k=1}kd^{k-1})}{(\sum^n_{k=1}d^k)^2} \\
    &= \frac{\sum_{k=1}^nd^{2k-1} - \sum_{k=1}^nkd^{2k-1}}{(\sum^n_{k=1}d^k)^2} \\
    &= \frac{\sum_{k-1}^nd^{2k-1}-\sum_{k=1}^nkd^{2k-1}}{(\sum^n_{k=1}d^k)^2} \\
    &= \frac{\sum_{k=1}^n(1-k)d^{2k-1}}{(\sum^n_{k=1}d^k)^2} \\
    &= \frac{\sum^{n}_{k=2}(1-k)d^{2k-1}}{(\sum^n_{k=1}d^k)^2} < 0.
    \end{aligned}
\end{equation}

From the derivative of $M$, we can see the increase of $d$ (i.e., increasing sparsity $T_{\mathcal{G}}$) leads to the monotonical decrease of $M$, which causes the drop of the lower bound of $\mathcal{P}_n(i)$. Thus, we prove Property \ref{theorem d} and \ref{theorem t}.

% We first shows the partial derivative with respect to $d$ for the lower bound of the numerator of $\mathcal{P}_n(i)$:
% \begin{equation}
%     \begin{aligned}
%         \frac{\partial f^{\mathcal{P}_n(i)}_{num}}{\partial d} &= P  + 2 d P^2 + 3 d^2 P^3 + \cdots + nd^{n-1}P^{n}.
%     \end{aligned}
% \end{equation}
% Then we presents the partial derivative with respect to $d$ for the denominator of the lower bound of $\mathcal{P}_n(i)$ as follows:
% \begin{equation}
%     \begin{aligned}
%     f^{\mathcal{P}_n(i)}_{den} &= \sum_{k=1}^n d^k \\
%     \frac{\partial f^{\mathcal{P}_n(i)}_{den}}{\partial d} &= 1 + 2d + 3d^2 + \cdots + nd^{n-1},
%     \end{aligned}
% \end{equation}
% where $f^{\mathcal{P}_n(i)}_{den}$ means the denominator for $\mathcal{P}_n(i)$. As $0 \leq \mathcal{P}_1(i) \leq 1$, we can easily observe that $\frac{\partial f^{\mathcal{P}_n(i)}_{num}}{\partial d} < \frac{\partial f^{\mathcal{P}_n(i)}_{den}}{\partial d}$. Thus, when $d$ increases, the denominator of the lower bound of $\mathcal{P}_n(i)$ will always increment larger than the numerator, which leads to the drop of the lower bound of $\mathcal{P}_n(i)$. Here, we have proved theorem \ref{theorem d}.
\end{proof}

Though different graphs require different contextual scope, aforementioned GCL methods only have fixed contextual scope. In contrast, \textbf{\ourmethod~can adjust the contextual scope by changing $n$}, which allows us to choose the most suitable scale for datasets with different properties. Moreover, we can select $n$ guiding by the theoretical findings above.

\subsection{Guidance of Selecting $n$}
We consider the $n$ just before the break of strong homophily dominance (i.e., $\mathcal{P}_n(i) > 0.5$) as the selected $n$ for model training. Homophily-dominant neighbourhoods are more beneficial for GNN layers, since in such neighbourhoods the class label of each node may be determined by the majority of the class labels in the neighbourhood \cite{zhu2020beyond}. However, only meeting the homophily-dominant requirement may not be sufficient for generating high quality contextual representation. This is because homophily-dominant neighbourhood can still include too much abundant or noisy neighbouring information, i.e., neighbouring nodes with different classes. To ensure the neighbourhood aggregation is conducted with a majority of homophilic neighbours, we consider the break of strong homophily dominance as the condition to select $n$. Surprisingly, this approach provides a good guidance to the selection of $n$. The $n$ just before the break of strong homophily dominance of $\mathcal{P}_n(i)$ is consistent with the best $n$ for 4 out of 5 datasets. The experiment result is presented in Section \ref{n-th graph power}.

% and boosts the model performance.

\section{Related Work}
\subsection{Graph Neural Networks}
Firstly introduced in Scarselli's work \cite{scarselli2008graph}, GNNs aim to extend deep neural networks to handle graph-structured data. GNNs consist of two domains: spectral-based methods \cite{defferrard2016convolutional, henaff2015deep, bruna2013spectral, kipf2016semi} and spatial-based methods \cite{velivckovic2017graph, hamilton2017inductive, wu2019simplifying}. While spectral-based methods adopt spectral representation of graphs, spatial-based methods conduct feature aggregation based on nodes spatial neighbours (e.g., GAT \cite{velivckovic2017graph}). Notably, GCN \cite{velivckovic2017graph} bridges the gap between these two domains by approximating spectral-based convolution with the first order of Chebyshev polynomial filters. Spatial-based methods are currently more prosperous since they have advantages in efficiency and general applicability. For example, to further improve GCN, GAT \cite{velivckovic2017graph} presents an attention-based approach to weightly aggregate node neighbours representations. 
SGC \cite{chen2020simple} simplifies GCN by removing the non-linearity and collapsing weight matrices among graph convolution layers. However, most GNNs rely extensively on labelling information, whereas the collecting process is expensive. To address this issue, GCL emerged.
We proposed \ourmethod, which can generate effective node representations without labelling information.

\subsection{Contrastive Learning}
Contrastive learning is a self-supervised learning paradigm usually based on MI maximisation. It aims to maximise MI between similar data instances (e.g., the same object in different augmented views and representations of the same object in different scales)
\cite{liu2021graph, tan2022federated}. It has been successfully applied in image classification tasks with promising results. For example, Deep Infomax \cite{hjelm2018learning}, Moco \cite{he2020momentum}, and SimCLR \cite{chen2020simple} train image encoders by discriminating two augmented images. Recently, some works attempted to adapt this concept to GNNs. DGI \cite{velivckovic2018deep} borrows the MI maximisation idea from Deep Infomax \cite{hjelm2018learning} and builds contrastiveness by contrasting node- and graph-level contextual representations. MVGRL \cite{hassani2020contrastive} further enriches the contrastiveness by building contrastiveness between augmented views of graphs.

Different from DGI and MVGRL, GRACE \cite{zhu2020deep} and
GMI \cite{peng2020graph}
create contextual representation from the same scale, first-order neighbourhood, respectively. Though these GCL methods have achieved promising results, they still share several issues, including the fixed contextual scope and biased contextual representation. \ourmethod~addresses these problems as it can easily adjust the contextual scope and ensure contrastiveness is built within connected components.

\vspace{-1.7mm}

\section{Experiment}
% \subsection{Datasets}
\begin{table*}[t]
	\scriptsize
	\caption{Node classification results on 7 small to medium-sized datasets compared with 14 baselines. Here, the ``Data'' column indicates what kind of data the method need to use in training. \textbf{X, A} and \textbf{Y} means feature matrix, adjacency matrix, and label information, respectively. OOM represents out-of-memory. The best performance for each dataset is in \textbf{bold}.}
	\resizebox{18cm}{!}{
	\begin{tabular*}{1.0\textwidth}{@{\extracolsep{\fill}} l|l|ccccccc}
		\toprule
		\textbf{Data} & \textbf{Method} & \textbf{Cora} & \textbf{CiteSeer} & \textbf{PubMed} & \tabincell{c}{\textbf{Coauthor} \\ \textbf{Physics}} & \tabincell{c}{\textbf{Coauthor} \\ \textbf{CS}} & \tabincell{c}{\textbf{Amazon}\\\textbf{Computers}} & \tabincell{c}{\textbf{Amazon}\\\textbf{Photos}} \\
		\midrule
% % 		\textbf{A, Y} & LP & 68.0 & 45.3  & 63.0 & ?$\pm{?}$  & ?$\pm{?}$ & ?$\pm{?}$ \\
		\textbf{X, A, Y} & MLP & 56.1$\pm{0.3}$ & 56.9$\pm{0.4}$ & 71.4$\pm{0.1}$ & 93.5$\pm{0.1}$  & 90.4$\pm{0.1}$ & 73.9$\pm{0.1}$ & 78.5$\pm{0.1}$   \\
		\textbf{X, A, Y} & GCN & 81.5 & 70.3 & 79.0 & 95.7$\pm{0.2}$  & 93.0$\pm{0.3}$ & 86.3$\pm{0.5}$ & 87.3$\pm{1.0}$  \\
		\textbf{X, A, Y} & GAT & 83.0$\pm{0.7}$ & 72.5$\pm{0.7}$ & 79.0$\pm{0.3}$  & 95.5$\pm{0.2}$  & 92.3$\pm{0.2}$ & 87.1$\pm{0.4}$ & 86.2$\pm{1.5}$ \\
		\textbf{X, A, Y} & SGC & 81.0$\pm{0.0}$ & 71.9$\pm{0.1}$ & 78.9$\pm{0.0}$ & 95.8$\pm{0.1}$  & 92.7$\pm{0.1}$ & 74.4$\pm{0.1}$ & 86.4$\pm{0.0}$ \\
% 		\textbf{X, A, Y} & CG3 & 83.4$\pm{0.7}$ & 73.6$\pm{0.8}$ & 80.2$\pm{0.8}$ & OOM  & 92.3$\pm{0.2}$ & 79.9$\pm{0.6}$ & 89.4$\pm{0.5}$& OOM \\
        \midrule
		\textbf{X, A}	& DeepWalk & 69.5$\pm{0.6}$ & 58.8$\pm{0.6}$ & 69.9$\pm{1.3}$  & 91.8$\pm{0.2}$  & 84.6$\pm{0.2}$ & 85.7$\pm{0.1}$& 89.4$\pm{0.1}$ \\
		\textbf{X, A}	& Node2vec & 71.2$\pm{1.0}$ & 47.6$\pm{0.8}$ & 66.5$\pm{1.0}$  & 91.2$\pm{0.1}$  & 85.1$\pm{0.1}$ & 84.4$\pm{0.1}$& 89.7$\pm{0.1}$ \\
		\textbf{X, A}	& GAE & 71.1$\pm{0.4}$ & 65.2$\pm{0.4}$ & 71.7$\pm{0.9}$  & 94.9$\pm{0.1}$  & 90.0$\pm{0.7}$ & 85.3$\pm{0.2}$& 91.6$\pm{0.1}$ \\
		\textbf{X, A}	& VGAE & 79.8$\pm{0.9}$ & 66.8$\pm{0.4}$ & 77.2$\pm{0.3}$  & 94.5$\pm{0.1}$  & 92.1$\pm{0.1}$ & 86.4$\pm{0.2}$& 92.2$\pm{0.1}$ \\
		\midrule
		\textbf{X, A}	& DGI & 81.7$\pm{0.6}$ & 71.5$\pm{0.7}$ & 77.3$\pm{0.6}$  & 94.5$\pm{0.5}$  & 92.2$\pm{0.6}$ & 84.1$\pm{0.4}$& 91.5$\pm{0.3}$ \\
		\textbf{X, A}	& GMI & 82.7$\pm{0.2}$ & 73.0$\pm{0.3}$ & 80.1$\pm{0.2}$ & OOM & OOM  &  76.8$\pm{0.1}$ & 85.1$\pm{0.1}$ \\
		\textbf{X, A}	& MVGRL & 82.9$\pm{0.7}$ & 72.6$\pm{0.7}$ & 79.4$\pm{0.3}$  & 95.3$\pm{0.1}$  & 92.1$\pm{0.1}$ & 81.8$\pm{0.5}$ & 90.7$\pm{0.3}$ \\
		\textbf{X, A}	& GRACE & 80.0$\pm{0.4}$ & 71.7$\pm{0.6}$ & 79.5$\pm{1.1}$ & OOM & 92.8$\pm{0.1}$ & 87.2$\pm{0.4}$  & 92.7$\pm{0.3}$\\
% 		\textbf{X, A}	& SubG-Con & 83.5$\pm{0.5}$ & 73.2$\pm{0.2}$ &  81.0$\pm{0.1}$ & 96.2$\pm{0.4}$  & 94.0$\pm{0.6}$ & 82.2$\pm{1.0}$ & 90.0$\pm{0.7}$  \\
		\textbf{X, A}	& GCA & 80.4$\pm{0.4}$ & 71.2$\pm{0.2}$ &  80.4$\pm{0.8}$ & \textbf{95.9}$\pm{0.2}$  & 93.3$\pm{0.1}$ & 87.8$\pm{0.3}$ & 93.2$\pm{0.3}$  \\
		\textbf{X, A}	& BGRL & 81.1$\pm{0.2}$ & 71.6$\pm{0.4}$ &  80.0$\pm{0.4}$ & 95.8$\pm{0.4}$  & 93.3$\pm{0.4}$ & 88.9$\pm{0.3}$ & 93.2$\pm{0.3}$  \\
		\midrule
			\textbf{X, A}	& \textbf{\ourmethod} & \textbf{84.7}$\pm{0.3}$ & \textbf{74.1}$\pm{0.2}$& \textbf{81.6}$\pm{0.3}$ &  95.6$\pm{0.3}$  & \textbf{93.4}$\pm{0.3}$ &\textbf{90.1}$\pm{0.5}$& \textbf{93.8}$\pm{0.7}$ \\
		\bottomrule
	\end{tabular*}
	}
\label{tab: classification results}
\vspace{-6.5mm}
\end{table*}

\begin{table}[t]
\caption{The statistics of benchmark datasets. }
\scriptsize
	\centering
	\begin{tabular}{@{}lccccc@{}}
		\toprule
		\textbf{Dataset} &  \textbf{Nodes} & \textbf{Edges} & \textbf{Features} & \textbf{Classes}  \\ \midrule
		\textbf{Cora}              & 2,708    & 5,429        & 1,433            & 7            \\
		\textbf{CiteSeer}         & 3,327    & 4,732        & 3,703            & 6   \\
		\textbf{PubMed}           & 19,717   & 44,338      & 500               & 3      \\
		\textbf{Coauthor CS}     & 18,333     & 81,894       & 6,805              & 15        \\
		\textbf{Coauthor Physics}     & 34,493     & 991,848       & 8,415       & 5          \\
		\textbf{Amazon Computers}     & 13,752     & 245,861       & 767       & 10           \\
		\textbf{Amazon Photos}    & 7,650     & 119,081       & 745              & 8              \\
		\textbf{ogbn-arxiv}    & 169,343     & 1,166,243       & 128             & 40      \\
 \bottomrule
	\end{tabular}
	\vspace{-5.5mm}
	\label{tab:dataset}
\end{table}

\begin{table}[tbp]
\centering
\scriptsize
  \caption{Node classification result on ogbn-arxiv. %Except for \ourmethod, the other results are borrowed from ~\cite{zbontar2021barlow}. 
  }
  \label{tab:ogbn-products}
  \begin{tabularx}{0.8\linewidth}{l|p{2.3cm}<{\centering} p{1.5cm}<{\centering}}
    \toprule
    \textbf{Method} &\textbf{Valid} & \textbf{Test} \\
    \midrule
    MLP & 57.7$\pm{0.4}$ & 55.5$\pm{0.2}$ \\
    Supervised GCN & 73.0$\pm{0.2}$ & 71.7$\pm{0.3}$  \\
    \midrule
    Node2vec & 71.3$\pm{0.1}$ & 70.1$\pm{0.1}$ \\
    % DeepWalk & 87.4 $\pm{0.1}$ & 73.1 $\pm{0.4}$\\
    % DeepWalk + fts & 87.8 $\pm{0.1}$ & 73.4 $\pm{0.1}$\\
    DGI & 71.3$\pm{0.1}$ & 70.3$\pm{0.2}$ \\
    GRACE & 72.6$\pm{0.2}$ & 71.5$\pm{0.1}$ \\
    BGRL & 72.5$\pm{2.1}$ & 71.6$\pm{1.6}$ \\
    \midrule
    \textbf{\ourmethod}  & 72.6$\pm{0.4}$ & 71.4$\pm{0.6}$ \\
  \bottomrule
\end{tabularx}
\label{tab:ogbn}
\vspace{-6.5mm}
\end{table}

\subsection{Details of the Experiments}
To evaluate the effectiveness of our proposed method, we conducted extensive experiments on 8 benchmark datasets, including 5 citation networks (i.e., Cora, CiteSeer, PubMed, Coauthor CS, and Physics), 2 Amazon co-purchasing networks (i.e., Amazon Computers and Photos), and a large-scale dataset, ogbn-arxiv. The statistic of these datasets is summarised in Table \ref{tab:dataset}. For the first three networks, we adopt the same dataset split as \cite{yang2016revisiting}. For Coauthor and Amazon datasets, we randomly split these datasets, where 10\%, 10\% and the remaining nodes are chosen for training, validation and test set, respectively. For the large-scale dataset, ogbn-arxiv, we use the default setting as described in \cite{hu2020open}. 

In our experiment, we mainly tune three parameters: $n$-th power of $\textbf{A}$ , sample size $S$, and hidden size $D'$. Specifically, $n$ is selected from 1 to 20, while $S$ is chosen from 500 to 3000, with every increment by 500, for Cora and Citeseer, and 3000 to 10000, with every increment by 1000, for the remaining datasets. For $D'$, it is chosen from \{512, 1024, 2048, 4196, 8192\}. After tunning these parameters, the best performance of our model for each dataset is recorded in Table \ref{tab: classification results}. 

\begin{table*}[t]
    
    \footnotesize
    \caption{The evaluation of $n$-th power of $\textbf{A}$ on five benchmark datasets. $\mathcal{P}_n(i)$ is the contextual homophily rate for each $n$. The best performance for each dataset is in bold. The $n$ and the model performance just before the break of homophily dominance according to $\mathcal{P}_n(i)$ are underlined. ``Homo rate'' means the proportion of homophily edges, which connect nodes with the same class, on total number of edges in the graph. It equals to $\mathcal{P}_{1}(i)$. }
    \centering
    \scalebox{0.78}{
    \begin{tabular*}{1.25\textwidth}{lcccccccccccc}
    \toprule
    {\textbf{Dataset}} & \textbf{Sparsity $T_{\mathcal{G}}$} & \textbf{Homo Rate $\mathcal{P}_1(i)$} &  \textbf{1}  &    \textbf{2}  &    \textbf{3}  &    \textbf{4}  &    \textbf{5}  &    \textbf{6}  &    \textbf{7}  &    \textbf{8}  &    \textbf{9}  &    \textbf{10} \\
    \midrule
    \textbf{Cora}      &  0.074\%  &  81.0\%  &  80.1$\pm{0.2}$ &  82.0$\pm{0.1}$&  83.2$\pm{0.2}$&  83.9$\pm{0.3}$&  \underline{84.2$\pm{0.2}$} &  84.3$\pm{0.4}$&  84.4$\pm{0.3}$&  84.6$\pm{0.4}$&  \textbf{84.7}$\pm{\textbf{0.3}}$&  83.8$\pm{0.4}$\\
    \textbf{$\mathcal{P}_n(i)^{cora
    }$}  & -  &  -
    & 0.81 & 0.75 & 0.67 & 0.59 & \underline{0.52} & 0.45 & 0.39 & 0.34 & 0.29 & 0.25 \\
    \midrule
    \textbf{CiteSeer} &    0.042\%  &  72.6\%   &  74.0$\pm{0.2}$&  \underline{\textbf{74.1}$\pm{\textbf{0.2}}$}&  73.9$\pm{0.3}$&  73.7$\pm{0.3}$&  73.8$\pm{0.4}$& 
    73.7$\pm{0.2}$&  73.7$\pm{0.3}$&  73.7$\pm{0.2}$&  73.6$\pm{0.3}$&  73.5$\pm{0.2}$\\
    \textbf{$\mathcal{P}_n(i)^{citeseer}$}  & -  &  -
    & 0.73 & \underline{0.6} & 0.48 & 0.38 & 0.3 & 0.23 & 0.17 & 0.13 & 0.1 & 0.07 \\
    \midrule
    \textbf{PubMed}   &   0.011\%    &  80.2\%  &  76.6$\pm{0.4}$&  78.6$\pm{0.3}$&  79.4$\pm{0.4}$&  80.1$\pm{0.4}$&  80.4$\pm{0.2}$ &  80.4$\pm{0.4}$ &  79.8$\pm{0.3}$ &  81.3$\pm{0.2}$ &  81.3$\pm{0.1}$ &  \underline{\textbf{81.4}$\pm{\textbf{0.2}}$} \\
    \textbf{$\mathcal{P}_n(i)^{pubmed}$}  & -  &  -
    & 0.80 & 0.73 & 0.68 & 0.63 & 0.59 & 0.56 & 0.54 & 0.52 & 0.51 & \underline{0.50} \\
    \midrule
    \textbf{Amazon Computers} & 0.130\%  &  77.7\%  &  89.9$\pm{0.5}$ &  \underline{\textbf{90.1}$\pm{\textbf{0.8}}$} &  89.1$\pm{1.1}$ &  89.4$\pm{1.0}$ &  89.3$\pm{1.1}$ &  87.7$\pm{0.7}$&  87.4$\pm{1.2}$& 87.1$\pm{0.9}$ &  87.0$\pm{0.8}$ &  87.2$\pm{0.9}$ \\
    \textbf{$\mathcal{P}_n(i)^{ac}$}  & -  &  -
    & 0.78 & \underline{0.62} & 0.48 & 0.37 & 0.29 & 0.23 & 0.18 & 0.14 & 0.11 & 0.08 \\
    \midrule
    \textbf{Amazon Photos}&  0.203\%   &  82.7\%  &  92.5$\pm{0.6}$ &  93.4$\pm{0.4}$&  \underline{\textbf{93.8}$\pm{\textbf{0.4}}$} &  93.4$\pm{0.3}$ &  93.0$\pm{0.9}$ &  92.9$\pm{0.5}$ &  92.9$\pm{1.3}$ &  92.0$\pm{1.5}$&  92.3$\pm{1.2}$ &  91.6$\pm{0.6}$ \\
    \textbf{$\mathcal{P}_n(i)^{ap}$}  & -  &  -
   & 0.83 & 0.71 & \underline{0.59} & 0.49 & 0.4 & 0.33 & 0.28 & 0.23 & 0.19 & 0.16 \\
    \bottomrule
    \end{tabular*}
    }
    \vspace{-6.5mm}
    \label{tab:sparsity}
\end{table*}

\subsection{Node Classification Results}
\label{sec:node classification}
We choose 14 baselines to be compared with \ourmethod\ on node classification tasks. These baselines consist of MLP and three types of GNNs: supervised-, conventional self-supervised, and GCL approaches. For supervised GNNs, we select three widely-adopted supervised GNNs, which are GCN \cite{kipf2016semi}, GAT \cite{velivckovic2017graph}, SGC \cite{wu2019simplifying}. 
% These methods rely on labelling information during the training phase. 
Four conventional self-supervised methods including DeepWalk~\cite{perozzi2014deepwalk}, Node2vec~\cite{grover2016node2vec}, GAE~\cite{kipf2016variational} and VGAE~\cite{kipf2016variational}, and six GCL methods including DGI~\cite{velivckovic2018deep}, GMI~ \cite{peng2020graph}, MVGRL~\cite{hassani2020contrastive}, GRACE~\cite{zhu2020deep}, GCA~\cite{zhu2021graph}, and BGRL~\cite{thakoor2021bootstrapped} are chosen to be compared with our model. 

We run all baselines and our model on each small to medium-sized dataset five times, and the average node classification accuracy and associated standard deviation are reported in Table \ref{tab: classification results}. The table shows that \ourmethod\  achieved the best performance on 6 out of 7 small to medium-sized datasets. Notably, \ourmethod\ surpasses its self-supervised counterparts by 2.1\% in Cora, 1.1\% in CiteSeer, and 1.2\% in Amazon Computers. 

In addition, we compare \ourmethod\ with supervised methods (i.e., MLP and Supervised GCN) and self-supervised methods, including Node2vec, DGI, GRACE, and BGRL on ogbn-arxiv. The other GCL methods are not selected as they encounter out-of-memory issue during training. The experiment results are presented in Table \ref{tab:ogbn}, where we can observe \ourmethod\ has achieved on-par performance with the most competitive baselines (i.e., GRACE and BGRL). Except for \ourmethod, others results in the table are sourced from ~\cite{thakoor2021bootstrapped}

\ourmethod\ generally achieves promising results on benchmark datasets because with $n$-th power of $\textbf{A}$, \ourmethod~ can tune the contextual scope 
% of contextual representations, 
while most GCL baselines can only contrast to a fixed scope. Additionally, \ourmethod\ ensures the contrastiveness is established within connected components to reduce the bias of contextual representations. These advantages allow \ourmethod\ to focus on the suitable scale and lead to model performance improvement.
% from a suitable scale and reduce the bias of contextual representations.

\begin{table}
\vspace{-1mm}
\caption{Ablation study of \ourmethod.} 
	\scriptsize
	\begin{tabularx}{1.0\linewidth}{l|p{1.9cm}<{\centering} p{1.9cm}<{\centering} p{1.9cm}<{\centering}}
		\toprule
		\textbf{Method} &  \textbf{Cora} & \textbf{Citeseer} & \textbf{Pubmed}  \\ \midrule
		$\textbf{\ourmethod}_{\rm mean}$   & 79.1$\pm{0.5}$ & 70.5$\pm{0.5}$  & 77.2$\pm{0.5}$ \\
		$\textbf{\ourmethod}_{\rm smooth}$   & 81.3$\pm{0.5}$ & 71.9$\pm{0.5}$  & 80.8$\pm{0.5}$ \\
		$\textbf{\ourmethod}_{\rm sig}$  & 84.2$\pm{0.3}$   & 72.2$\pm{0.4}$  & 81.1$\pm{0.2}$ \\
		$\textbf{\ourmethod}_{\rm w/o\ sam}$ & 84.1$\pm{0.5}$ & 73.2$\pm{0.4}$ & 80.8$\pm{1.0}$ \\
		$\textbf{\ourmethod}_{\rm w/o\ p}$  & 80.8$\pm{0.6}$   & 71.4$\pm{0.2}$  & 79.5$\pm{0.4}$ \\
		\midrule
		\textbf{\ourmethod} & \textbf{84.7}$\pm{0.3}$ &\textbf{74.1}$\pm{0.2}$ & \textbf{81.6}$\pm{0.3}$  \\
 \bottomrule
	\end{tabularx}
	 \label{tab:ablation study}
	\vspace{-4mm}
% 	\label{tab:ablation study}
\end{table} 

\begin{figure}
    \centering
    \includegraphics[scale = 0.23]{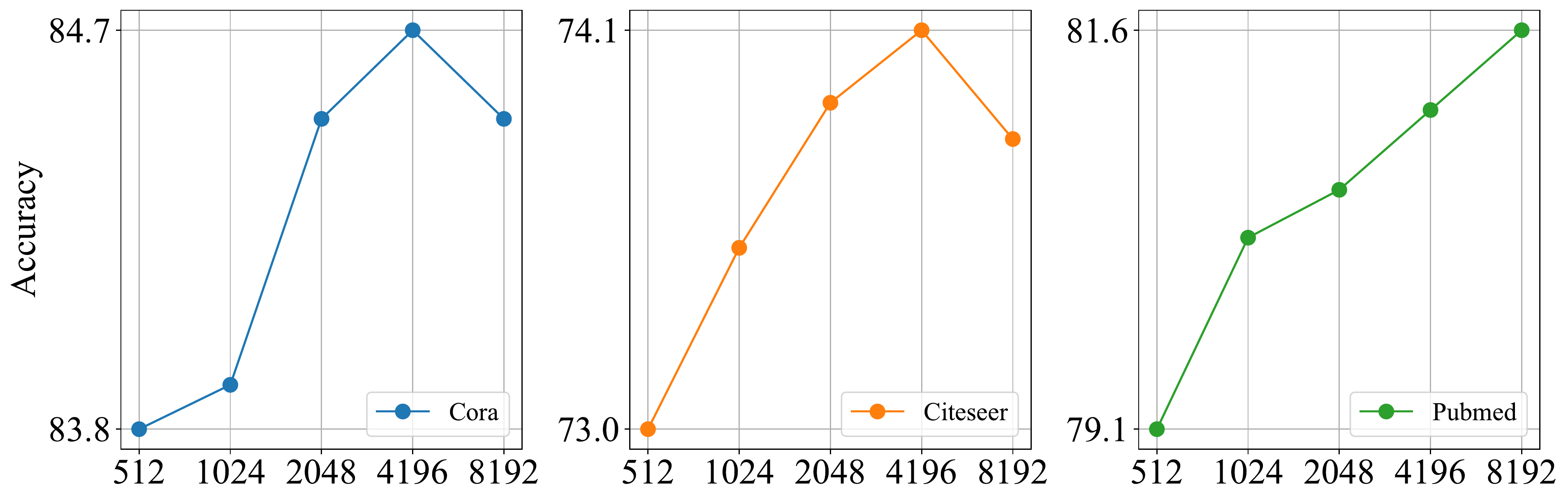}
    \caption{Parameter analysis for hidden size $D$ on Cora, CiteSeer, and PubMed.}
    \label{fig:parameter}
    \vspace{-5.5mm}
\end{figure}

\begin{figure*}
\centering
\subfigure[Cora]{\includegraphics[width=4.5cm]{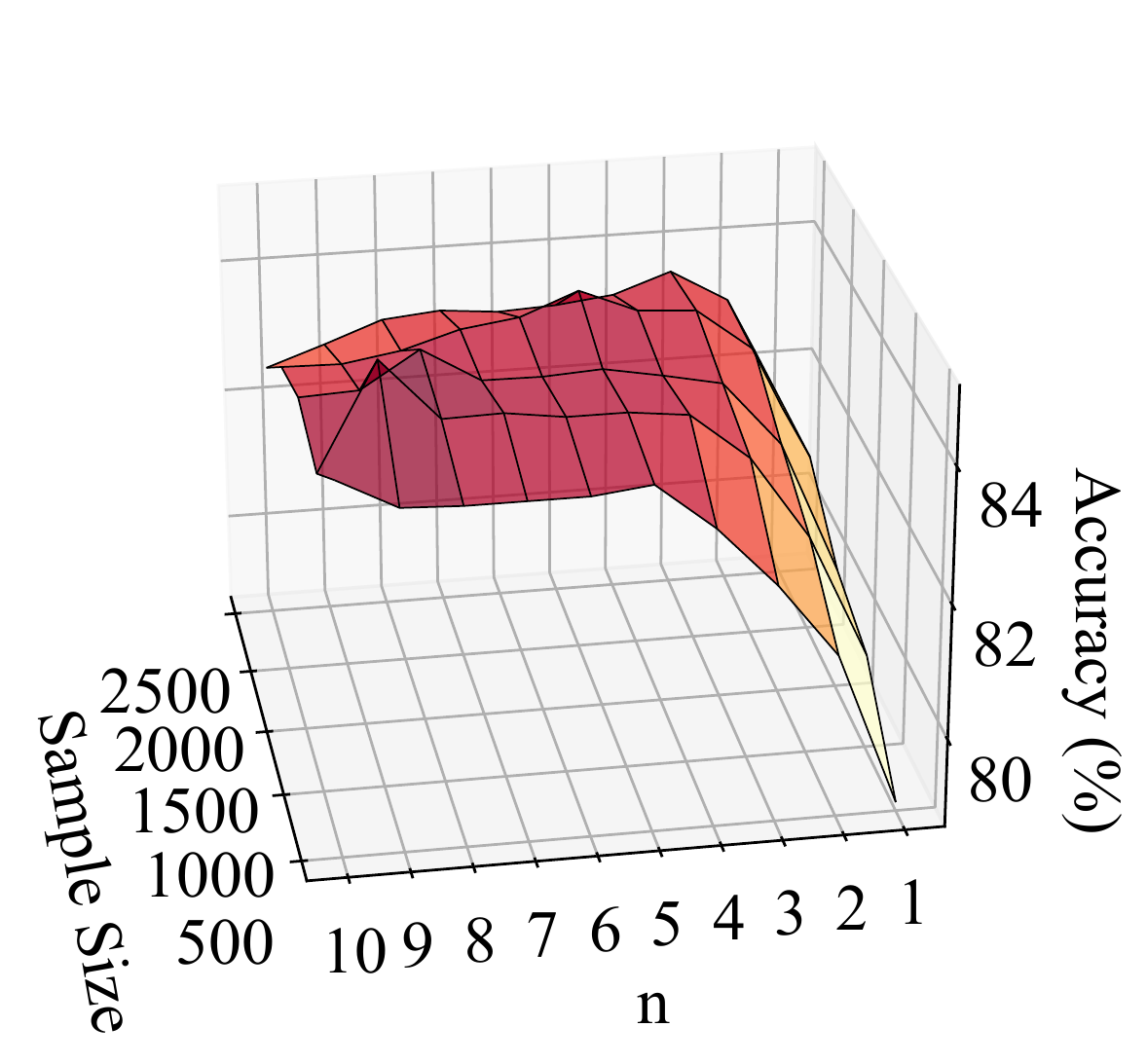}}
\hfill
\subfigure[CiteSeer]{\includegraphics[width=4.5cm]{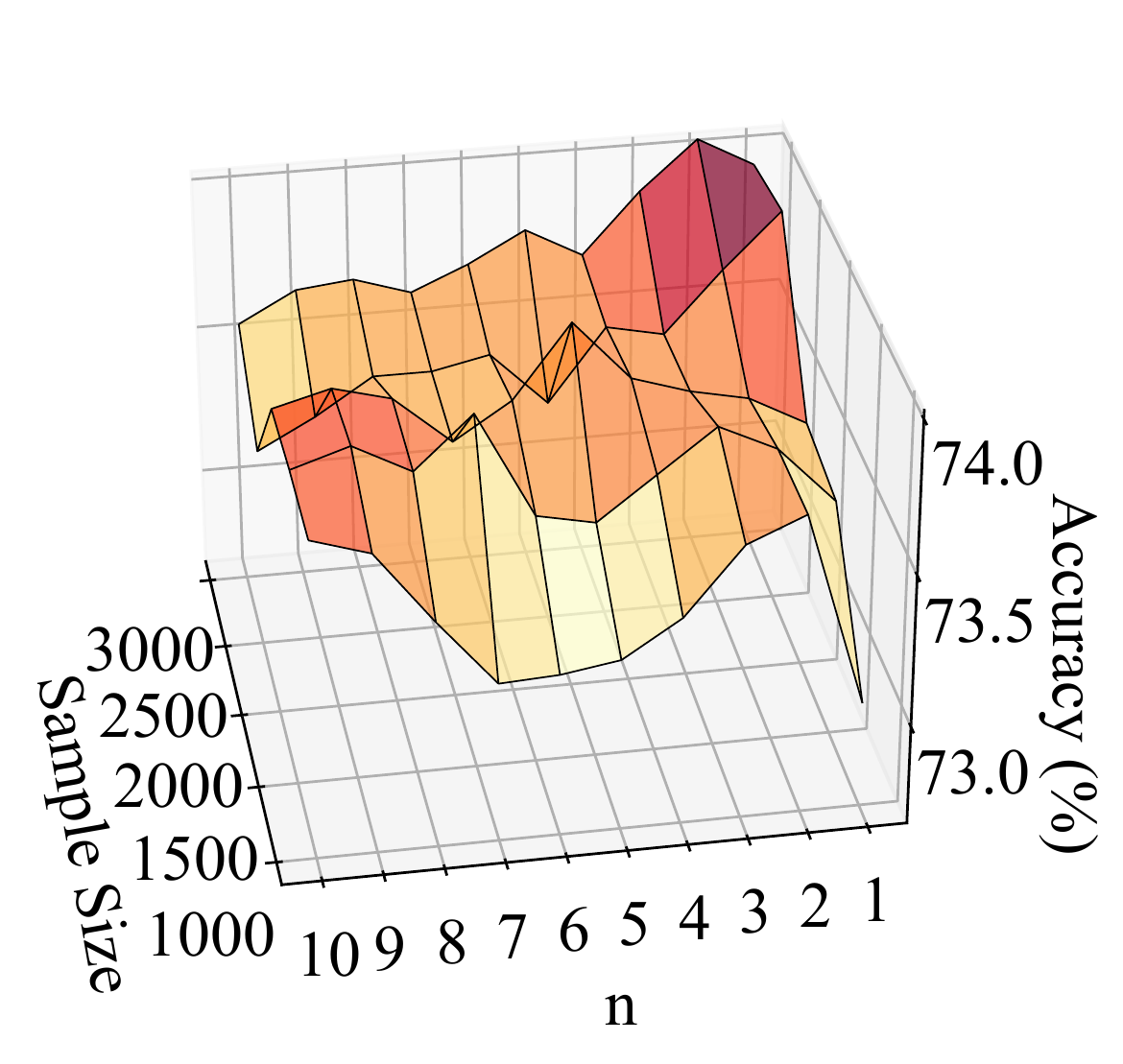}}
\hfill
\subfigure[PubMed]{\includegraphics[width=4.5cm]{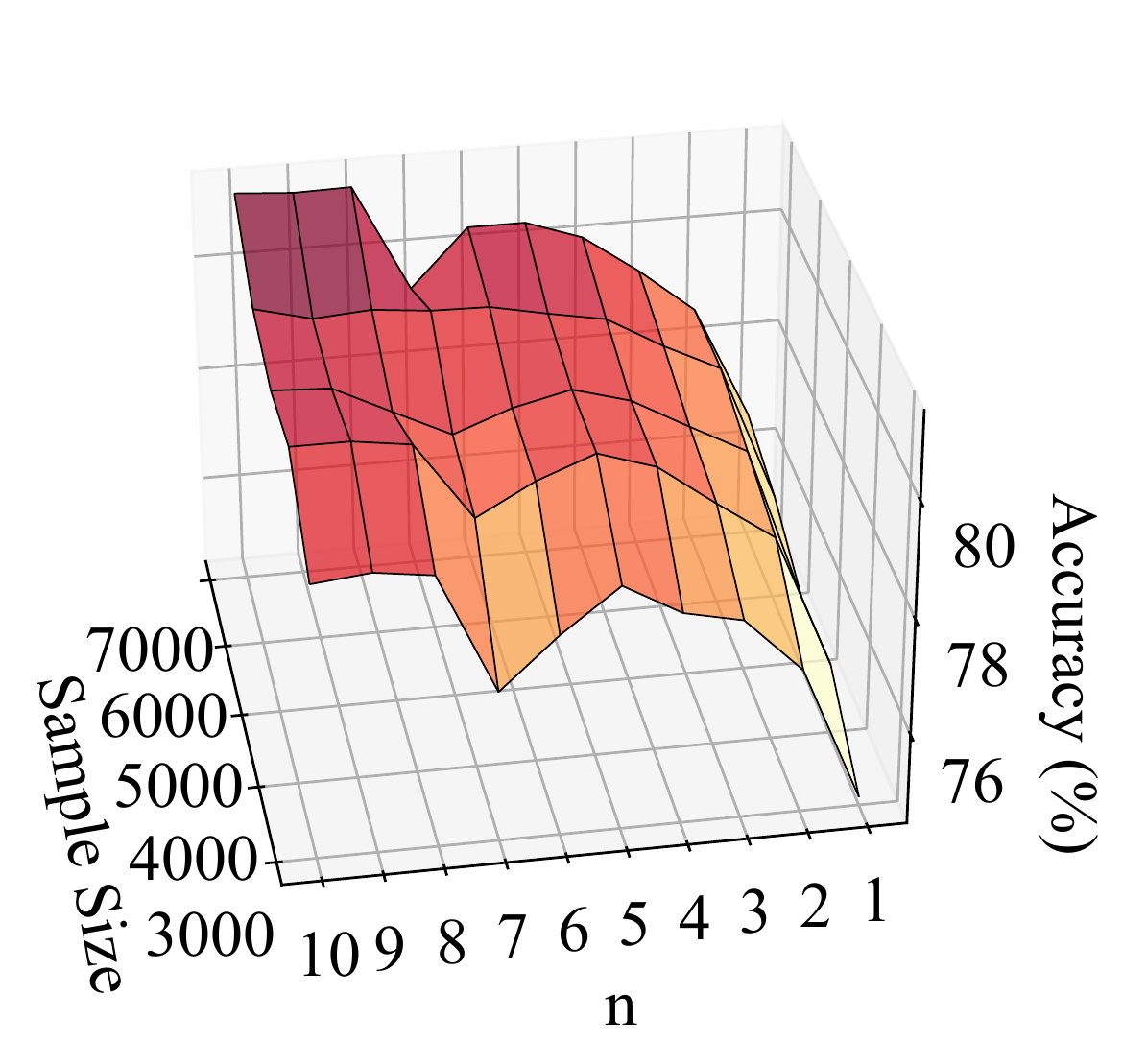}}
\caption{Joint parameter analysis of sample size and $n$-th power of $\textbf{A}$.}
\label{fig:joint parameter}
\vspace{-6.5mm}
\end{figure*}

\subsection{Ablation Study}
\label{ablation study section}
In this section, we compare the performance of our original method to its five variants: $\ourmethod_{\rm mean}$, $\ourmethod_{\rm smooth}$, $\ourmethod_{\rm sig}$, $\ourmethod_{\rm sam}$, and $\ourmethod_{\rm w/o\ p}$ on Cora, CiteSeer, and PubMed. The comparative results have been exhibited in Table \ref{tab:ablation study}.

For $\ourmethod_{\rm mean}$, the graph-level embedding is generated with the naive mean pooling approach, whereas $\ourmethod_{\rm smooth}$ uses 100-th power of $\textbf{A}$ to obtain oversmoothed embeddings which summarise all nodes information within a connected component. The method $\ourmethod_{\rm smooth}$ consistently outperforms $\ourmethod_{\rm mean}$, which indicates that establishing the contrastiveness within connected components instead of the whole graph is effective.

$\ourmethod_{\rm sig}$ uses a single encoder for both local and contextual view establishment, while $\ourmethod$ utilizes an additional auxiliary encoder for generating the contextual view. This auxiliary encoder targets to embed contextual information for better contextual representations generation. $\ourmethod_{\rm w/o\ sam}$ have no subsampling, while $\ourmethod$ uses subsampling as an augmentation to both increase the difficulty of the contrastive learning tasks and extend the scalability.
$\ourmethod_{\rm w/o\ p}$ removes the power mechanism for the contextual view, whereas $\ourmethod$ employs this mechanism to control the contextual scope of contextual representations. To shed light on the contributions of the auxiliary GNN encoder, subsampling, and the power mechanism, we compare $\ourmethod$ with $\ourmethod_{\rm w/o\ sam}$, $\ourmethod_{\rm sig}$, and $\ourmethod_{\rm w/o\ p}$ on three benchmark datasets. It is apparent that the model performance degrades without any of the three mechanisms mentioned above, which validates the effectiveness of these mechanisms. 

% \subsection{Parameter Sensitivity}
% We perform the sensitivity analysis on three parameters (i.e., $n$-th graph power, sample size $S$, and hidden size $D$) of \ourmethod\ on Cora, CiteSeer, and PubMed. The experimental results are reported in Table \ref{tab:sparsity} and Figure \ref{fig:parameter}.

\subsection{Parameter Study}

\subsubsection{$\textbf{n}$-th power of $\textbf{A}$}
\label{n-th graph power}
$n$ is a key parameter used to choose the $n$-th power of $\textbf{A}$ for contextual view generation. To evaluate the effect of $n$, we run \ourmethod\ with $n$ ranging from 1 to 10. Based on Equation \ref{low bound p}, we can calculate the lower bound of $\mathcal{P}_n(i)$ for each $n$ with $d$ and $\mathcal{P}_1(i)$. Here, we assume $\mathcal{P}_1(i)$ equals to the homophily rate of the dataset.

The model performance and the value of $\mathcal{P}_n(i)$ on 5 datasets are reported in Table \ref{tab:sparsity}. Here, we adopt the best parameter settings for each dataset except for $n$. Specifically, the chosen sample size $S$ are 1000, 3000, 7000, 10000, and 5000 for Cora, CiteSeer, PubMed, Computers and Photos, respectively. 
% It is worth noting that our theoretical explanation in Section \ref{sec theoretical justification} provides informative guidance for selecting $n$.
% if $d$ and $\mathcal{P}_1$ are available. 
As shown in the underlined results in Table \ref{tab:sparsity}, we can see the $n$ just before the break of strong homophily dominance (i.e., $\mathcal{P}_n(i) > 0.5$) for the lower bound of $\mathcal{P}_n(i)$ is consistent with the best $n$ for 4 out of 5 datasets. Though for Cora, the selected $n$ is not optimal, it still achieves state-of-the-art performance (84.2\%) compared with baselines.

\subsubsection{Hidden Size $D'$}
Hidden size $D'$ controls the dimensionality of hidden layers in the GNN encoder, and we change $D'$ from 512 to 8192 to see its effect on the model performance. The experiment results of the parameter analysis are shown in Figure \ref{fig:parameter}. The model performance on a larger dataset (i.e., PubMed) grows consistently when $D'$ increases, whereas the other two datasets achieve the highest performance at first and then degrade. We conjecture this is because a large $D'$ with too many parameters may overfit on small datasets.

\subsubsection{Joint Influence of Sample Size $S$ and $n$}
In this section, we explore the joint influence of sample sizes $S$ and $n$ on three datasets: Cora, CiteSeer, and PubMed. Specifically, we choose $S$ ranging from 500 to 2500 for Cora, 1000 to 3000 for CiteSeer with every increment by 500, and 3000 to 7000 for PubMed with every increment by 1000. The experiment result is presented in Figure \ref{fig:joint parameter}. For Cora and CiteSeer, we can observe that when $S$ increases, the model performance peaks at lower $n$. For PubMed, the model performance all peaks when $n$ is 10. 

This experiment result is consistent with our theoretical findings in Section \ref{sec theoretical justification}. We conjecture this phenomenon is because when $S$ increases, the average node degree $d$ would increase for the generated subgraph $\hat{\mathcal{G}}$ as the average node degree in $\hat{\mathcal{G}}$ is $\frac{S}{N}d$. According to Property \ref{theorem d}, when $d$ increases, the lower bound of $\mathcal{P}_n$ drops and leads to the break of homophily dominance easily when $n$ increases. Thus, the optimal $n$ decreases with the growth of $S$. For PubMed, from Table \ref{tab:sparsity}, we can see even with 7000 for $S$, the homophily dominance still holds when $n$ is 10. Thus, it is reasonable that the model performance all peaks at $n = 10$.

% Moreover, sample size $s$ not  seems to affect the selection of $n$ noticeably. For each selected sample size, the peak performance is achieved around $6$ and $7$, respectively. We conjecture this is because of the following proposition:

\begin{table}[tbp]
\scriptsize
  \caption{$n$-th power of $\textbf{A}$ computation time in \underline{seconds} on five datasets. Number in bracket means the number of $n$ used for the dataset. 
%   `Computers' and `Photos' mean Amazon Computers and Amazon Photos.
  }
  \label{tab:ogbn-paper}
    \centering
  \begin{tabularx}{0.9\linewidth}{ccccc}
    \toprule
    \textbf{Cora(7)} &\textbf{CiteSeer(6)} & \textbf{PubMed(10)} &\textbf{Computers(1)} &\textbf{Photos(2)}\\
    \midrule
    2.6e-04 & 2.2e-04 & 3.7e-04 & 2.7e-04 & 1.3e-03\\
  \bottomrule
  \label{graph power}
\end{tabularx}
\vspace{-9mm}
\end{table}
\subsection{Computation for $n$-th Power of $\textbf{A}$}
\label{sec:graph power}
To show the easiness of the computation for $n$-th Power of $\textbf{A}$ in \ourmethod, we run experiments on five datasets (i.e., Cora, CiteSeer, PubMed, Amazon Computers, and Amazon Photos) and report the average computation time per epoch in seconds for this operation in Table \ref{graph power}. From the table, we can observe that the computation is trivial in the training process.

% From Figure \ref{fig:my_label}, we can see that when $D$ increases, our model performance on PubMed grows consistently, whereas the other two datasets achieve the highest performance at first and then degrade. 

% There is a clear upward trend when $P$ grow for the model performance (Figure \ref{fig:n}), which is attributed to the inclusion of more information for node representation learning.

% \subsubsection{Hidden Size $D$.}

% \begin{wraptable}{r}{0.45\textwidth}
%     \vspace{-0.6cm}
%     \begin{minipage}{0.45\textwidth}
%     	\footnotesize
% \caption{graph power Computation time in seconds. Number in brackets means the n-th graph power.} 
% \label{tab:ablation}
% \vspace{0.1cm}
% \resizebox{1\columnwidth}{!}{
%   \begin{tabularx}{1.5\linewidth}{lp{1.6cm}<{\centering} p{1.6cm}<{\centering} p{1.6cm}<{\centering} p{1.6cm}<{\centering}}
%     \toprule
%     \textbf{Cora(7)} &\textbf{CiteSeer(6)} & \textbf{PubMed(10)} &\textbf{Computers(1)} &\textbf{Photo(2)}\\
%     \midrule
%     0.00026 & 0.00022 & 0.00037 & 0.00027 & 0.00132\\
%   \bottomrule
%   \label{graph power}
%   \vspace{-8mm}
% \end{tabularx}}
%     \end{minipage}
%     \vspace{-0.2cm}
% \end{wraptable}
% \subsection{graph power Computation}
% \label{sec:graph power}
% To show the easiness of graph power computation in \ourmethod, we run experiments on five datasets (i.e.,Cora, CiteSeer, PubMed, Amazon Computers, and Amazon Photo) and report the average computation time per epoch in seconds for this operation in Table \ref{graph power}. From the table, we can observe that the computation is trivial.

\section{Conclusion}
In this paper, we \yx{propose} a novel GCL approach, namely \ourmethod. We design a cross-scale contrastiveness to fuel the GNN encoder learning process by discriminating node representations in the patch- and contextual view. The proposed power mechanism allows our method to adjust the contextual scope when building contrastiveness and ensures the contrastiveness is established within connected components. These advantages allow \ourmethod\  to conduct a more fine-grained contrastiveness than the naive pooling approach and reduce the bias of generated contextual representations. Moreover, the architecture of \ourmethod~ can be considered as a unified framework to interpret existing GCL methods. Extensive experiments validate the effectiveness of our proposed approach in node classification tasks.

\section*{Acknowledgment}
This work was partially supported by an Australian Research Council (ARC) Future Fellowship (FT210100097).

% Experiments on various datasets with different size validate the effectiveness of our proposed approach in node classification tasks.

\appendix
\vspace{-1mm}
\section{Proof of Theorem 1.}
We first provide the Lemma \ref{lemma norm}, the definition of subspace, and Lemma \ref{lemma orth}:

\begin{lemma}
\textit{Given an adjacency matrix $\normalfont{\textbf{A}}$, its normalized augmented adjacency is $\normalfont{\hat{\textbf{A}} = \hat{\textbf{D}}^{-\frac{1}{2}}(\textbf{A} + \textbf{I})\hat{\textbf{D}}^{-\frac{1}{2}}}$, where $\normalfont{\hat{\textbf{D}} = \textbf{D} + \textbf{I}}$, and $\normalfont{\textbf{I}}$ is the identity matrix.  $\normalfont{\hat{\textbf{A}}}$ is symmetric with real eigenvalues $\lambda_1 \leq \lambda_2 \leq ... \lambda_N$, which have been sorted ascendingly. If the algebraic multiplicity of the largest eigenvalue $\lambda_N$ is $K \leq N$, which means the top $K$ eigenvalue are the same, we have the following properties:}
\begin{itemize}
    \item $\lambda_{N-K+1}, \lambda_{N-K+2} \cdots \lambda_{N}$ = 1, i.e., the $K$ largest eigenvalue equals to 1;
    \item $\lambda_{N-K} < 1; \lambda_{1} > -1$, i.e., the second largest eigenvalue is smaller than 1, while the smallest eigenvalue is larger than -1;
    \item The multiplicity $K$ is the number of connected components in the Graph $\mathcal{G}$ with the adjacency matrix $\normalfont{\textbf{A}}$. For each connected component, we have the eigenvector $\hat{v}_k:=\normalfont{\hat{\textbf{D}}^{\frac{1}{2}}u_k}$ corresponding to the eigenvalue $\lambda_{N-K}$, where $u_k \in \mathbb{R}^N$ indicates whether a node is belong to the $K$-th component. 
\end{itemize}
\label{lemma norm}
\end{lemma}

\begin{definition}[Subspace]
We define the subspace $\mathcal{M} \in \mathbb{R}^{N \times D}$ by $\mathcal{M} := \{\normalfont{\textbf{H} \in \mathbb{R}^{N \times D}|\textbf{H} = \hat{\textbf{V}}\textbf{M}, \textbf{M} \in \mathbb{R}^{K \times D}\}}$, where $\normalfont{\hat{\textbf{V}} \in \mathbb{R}^{N \times K}}$ is a collection of eigenvectors $\hat{v}_k$ of the largest eigenvalue of $\normalfont{\hat{\textbf{A}}}$ in Theorem 1.
\end{definition}

% To prove the convergence, which as $n$ increases, the output of the mechanism will approach a subspace gradually, we utilize the augmented spectral property \cite{oono2019graph}, and give the definition of the subspace as below:

\begin{lemma}
\label{lem:norm_adj}
\textit{Given a normalized adjacency matrix $\normalfont{\hat{\textbf{A}} \in \mathbb{R}^{N \times N}}$, a feature matrix $\normalfont{\textbf{H} \in \mathbb{R}^{N \times D}}$, the projection matrix for $\mathcal{M}$, $\normalfont{\hat{\textbf{V}}\hat{\textbf{V}}^T}$ , where $\normalfont{\hat{\textbf{V}}}$ is the normalized bases of $\mathcal{M}$, and $\normalfont{\hat{\textbf{F}}}$ is the orthogonal complement of $\normalfont{\hat{\textbf{V}}}$}, we have:
\vspace{-1mm}
\begin{equation}
    \begin{aligned}
    \normalfont{d_{\mathcal{M}}(\textbf{H})} &= \normalfont{\parallel\hat{\textbf{F}}^T\textbf{H}\parallel_F},\\
    \normalfont{d_{\mathcal{M}}(\hat{\textbf{A}}\textbf{H})} &= \normalfont{\parallel\Lambda\hat{\textbf{F}}^T \textbf{H}\parallel_F}, \\
    % d_{\mathcal{M}}(\hat{\textbf{A}}\textbf{H}) &\leq \lambda d_{\mathcal{M}}(\textbf{H}),\\
     &\leq 
    \normalfont{\parallel \Lambda \parallel_F \parallel \hat{\textbf{F}}^T \textbf{H} \parallel_F},
\end{aligned}
\label{eq:norm_adj}
\end{equation}

\noindent where $d_{\mathcal{M}}(\cdot)$ is the distance between representations and the subspace $\mathcal{M}$. The distance between node representations $\normalfont{\textbf{H}}$ and $\mathcal{M}$ is denoted as $d_{\mathcal{M}}(\normalfont{\textbf{H}}) = \rm inf_{p\in \mathcal{M}} \parallel \textbf{H} - p \parallel_F$. $\Lambda$ denotes all eigenvalues excluding the $K$ largest eigenvalues, and $\parallel \cdot \parallel_F$ represents the Frobenius norm. 
\label{lemma orth}
\end{lemma}

\begin{proof}
Lemma \ref{lemma norm} has been proved by  \cite{oono2019graph} to show augmented spectral property of an augmented adjacency, while Lemma 2 has been proved by \cite{huang2020tackling} based on the notion of projection. A projection matrix can project a given vector or matrix onto subspace to obtain the projected vector or matrix. By utilising Equation (\ref{eq:norm_adj}) in Lemma \ref{lem:norm_adj}, we will have the following derivation:% derive that
\vspace{-1mm}
\begin{equation}
    \begin{aligned}
    d_{\mathcal{M}}(\hat{\textbf{A}}^{n}\textbf{H}) &= \parallel\Lambda^{n}\hat{\textbf{F}}^T \textbf{H}\parallel_F,\\
     &= \parallel\Lambda\Lambda^{n-1}\hat{\textbf{F}}^T \textbf{H}\parallel_F,\\
     &\leq \parallel\Lambda\parallel_F\parallel\Lambda^{n-1}\hat{\textbf{F}}^T \textbf{H}\parallel_F, \\
     &\leq \parallel\Lambda\parallel_F d_{\mathcal{M}}(\hat{\textbf{A}}^{n-1}\textbf{H}), \\
     &\leq \lambda d_{\mathcal{M}}(\hat{\textbf{A}}^{n-1}\textbf{H}).
    \end{aligned}
\end{equation}

Here, we get the inequality shown in Theorem 1.

\end{proof}

\vspace{-7.5mm}

% \begin{displaymath} 
% \begin{aligned}
%     d_{\mathcal{M}}(\hat{\textbf{A}}^{n + 1}\textbf{H}) &=\  \parallel \textbf{F}^T(\hat{\textbf{A}}^{n + 1}\textbf{H}) \parallel_F \\
%     &=\ \parallel \Lambda\textbf{F}^T\textbf{H} \parallel_F \\ 
% \end{aligned}

\bibliographystyle{IEEEtran}
\bibliography{GCLGP}

\end{document}